\documentclass[pdflatex,sn-mathphys-ay]{sn-jnl}

\usepackage{blindtext}
\usepackage{graphicx}%
\usepackage{multirow}%
\usepackage{amsmath,amssymb,amsfonts}%
\usepackage{mathrsfs}%
\usepackage{stmaryrd}%
\usepackage[title]{appendix}%
\usepackage{xcolor}%
\usepackage{textcomp}%
\usepackage{booktabs}\let\cline\cmidrule
\usepackage{booktabs}%
\usepackage{hyperref}
\usepackage{algpseudocode}%
\usepackage[linesnumbered,ruled,vlined]{algorithm2e}
\usepackage{listings}%
\usepackage{nicefrac}%
\usepackage{tikz}
\usetikzlibrary{calc}
\usetikzlibrary{tikzmark}
\usepackage{subcaption}
\usepackage{longtable}%
\usepackage{array}
\usepackage{pgfplots}

\allowdisplaybreaks

\pgfplotsset{compat=1.18}

\newtheorem{theorem}{Theorem}
\newtheorem{lemma}[theorem]{Lemma} 
\newtheorem{proposition}[theorem]{Proposition} 
\newtheorem{remark}[theorem]{Remark}

\newtheorem{property}{Property}

\newcommand{\cX}{\mathcal{X}}

\newcommand{\cY}{\mathcal{Y}}
\newcommand{\cD}{\mathbf{D}}

\newcommand{\bX}{\mathbf{X}}

\newcommand{\bx}{\boldsymbol{x}}
\newcommand{\by}{\boldsymbol{y}}
\newcommand{\bh}{\boldsymbol{h}}
\newcommand{\bp}{\boldsymbol{p}}

\renewcommand{\vec}[1]{\boldsymbol{#1}}
\newcommand{\Prob}{\vec{p}}

\newcommand{\given}{\, | \,}
\newcommand*{\defeq}{\mathrel{\vcenter{\baselineskip0.5ex \lineskiplimit0pt
			\hbox{\footnotesize.}\hbox{\footnotesize.}}}%
	=}
\newcommand{\defi}{\defeq}

\newcommand{\fromto}{\longrightarrow}

\begin{document}

\title[Compact set-valued deep ensembling in\\ multi-class classification]{Compact set-valued deep ensembling in\\ multi-class classification}


\author[1]{\fnm{Kim-Dung} \sur{Tran}}\email{kim-dung.tran@hds.utc.fr}
\equalcont{These authors contributed equally to this work.}

\author[2]{\fnm{Dang-Man} \sur{Nguyen}}\email{mannd@jaist.ac.jp}
\equalcont{These authors contributed equally to this work.}

\author*[1]{\fnm{Vu-Linh} \sur{Nguyen}}\email{vu-linh.nguyen@hds.utc.fr}

\author[2]{\fnm{Xuan-Truong} \sur{Hoang}}\email{hxtruong@jaist.ac.jp}

\author[1]{\fnm{S\'ebastien} \sur{Destercke}}\email{sebastien.destercke@hds.utc.fr}

\author[2]{\fnm{Van-Nam} \sur{Huynh}}\email{huynh@jaist.ac.jp}

\affil*[1]{\orgdiv{UMR CNRS 7253, Heudiasyc}, \orgname{Universit\'e de Technologie de Compi\`egne}, \orgaddress{\city{Compi\`egne}, \country{France}}}

\affil[2]{\orgdiv{School of Knowledge Science}, \orgname{Japan Advanced Institute of Science and Technology}, \orgaddress{\city{Nomi}, \country{Japan}}}



\abstract{This paper tackles visible challenges in deep ensemble learning, where deep neural networks serve as ensemble members: training and storage burdens, and robustness of cautious (set-valued) predictions targeting multiple utilities, which may involve reward-sensitivity. To mitigate the training and storage burdens, we propose to employ compact ensembles, such as Bayesian Neural Networks and Convolutional Neural Networks with the Monte-Carlo dropout prediction option, to produce probabilistic predictions. For each query instance, these probabilistic predictions are then used to define a representative distribution optimizing some statistical distance. The representative distribution is then employed to define the Bayes-optimal prediction (BOP) of any utility. To address the potential unrobustness of singleton prediction making, we propose a family of set-utilities satisfying some desirable properties and whose set-valued BOPs can be found efficiently. Empirical evidence is then given to illustrate the potential (dis)advantages of the proposed ensemble learning framework.}


\maketitle

\section{Introduction}

Machine learning methods typically face multiple challenges or needs: obtaining high predictive power, being adaptable to different tasks without the need to retrain new models, and offering trustworthy predictions. In this paper, we propose (deep) probabilistic ensemble-based approaches combined with robust predictions to solve these issues. 

Ensemble learning is commonly used to boost the predictive performance and robustness of (probabilistic) predictors in multi-class classification (MCC), and in machine learning in general \citep{dietterich2000ensemble,nguyen2025credal,Zhou2021ensemblelearning}. Ensemble learning typically consists of two phases: Generation and aggregation. During the former, a set of ensemble members is generated, where a trade-off between their accuracy and diversity is sought. The latter is typically performed in an instance-wise manner, where the predictions of the ensemble members are aggregated into final predictions.

Despite a rather rich literature in ensemble learning, there remain visible challenges in both generation and aggregation phases when deploying ensemble learning in applications where deep neural networks serve as ensemble members. Regarding the generation phase, training and storage burdens are visible challenges when deploying conventional ensemble learning techniques, such as bagging, boosting, stacking, and random initialization \citep{yang2023survey}, which typically require one to train and maintain a reasonably large set of ensemble members to ensure diversity. 

Another fact is that modern machine learning models such as foundation models are trained to solve generic problems, without a pre-specified task for which they will be used. Their outputs should therefore be amenable to adapt to various tasks. Probabilistic predictions are ideal for that task, as one can then plug-in task-specific utility or reward matrix to maximize the expected utility. To obtain such a probabilistic prediction from ensembles, we propose to proceed in a two-step aggregation phase \citep{nguyen2025credal}: for each query instance, the probabilistic predictions of the ensemble members are aggregated into a representative probability distribution. Task-specific utilities can then be used to define the Bayes-optimal prediction (BOP) of any utility \citep{del2009learning,mortier2021efficient}.

Finally, in the case of high-stakes applications, one may wish to produce set-valued classification to avoid predicting a unique optimal prediction when this later one is not very robust to small changes in the produced probabilistic prediction. Indeed, obtaining reliable estimates of representative probability distributions may be challenging, especially in the presence of noisy training data or test data. One sound approach to gain robustness might be to produce set-valued BOPs of reward-sensitive set-utilities (See \citep{yang2016costs} and references therein), which satisfy various desirable properties, which reflect different aspects of the trade-off between recall (including the right class in the set-valued prediction) and precision (producing small sets as predictions). Unfortunately, we are unaware of reward-sensitive set-utilities that satisfy good properties and whose BOPs can be found efficiently. 

In this paper, we propose a compact cautious deep ensembling framework to tackle the mentioned challenges. Our contributions are as followed:
\begin{itemize}
\item to mitigate the training and storage burdens of the generation phase and to target the optimization of multiple utilities during the aggregation phase, we propose to employ compact deep ensembles, whose ensemble members are essentially generated in an instance-wise manner during the aggregation phase by distorting a pre-trained deep neural network, to define the representative probability distribution from which a BOP of any utility can be defined. Concrete examples are Convolutional Neural Networks (CNNs) with the Monte-Carlo dropout prediction option \citep{lemay2022improving} and Bayesian Neural Networks (BNNs) \citep{foong2020expressiveness,jospin2022hands};
\item to mitigate the potential unrobustness of singleton prediction making when taking into account reward-sensitivity, we propose a family of reward-sensitive utility-discounted accuracies, which satisfies a number of important properties discussed in \citep{yang2016costs} and references therein, and whose BOP can be found through computationally efficient procedures.
\end{itemize}

The remainder of this paper is organized as follows. Section \ref{sec:Probabilistic multi-class classification} recalls the basic setting of MCC, the notions of utilities and their BOPs. The compact cautious deep ensembling framework is presented in Section \ref{sec:Compact credal deep ensembling}. Section \ref{sec:Desirable properties} presents desirable properties and an analysis of the proposed family of reward-sensitive utility-discounted accuracies. Section \ref{sec:Experiments} provides empirical evidence to illustrate the potential (dis)advantages of the proposed ensemble learning framework on image data sets, where either the training set or the test data can contain noise. Finally, Section \ref{sec:Conclusion} concludes this paper and sketches out potential follow-up works. 

This paper is a significant extension of an earlier conference version \citep{tran2025robust}, in which we only used our aggregation approach on a limited number of data sets with classical, singleton-valued decision making. The current version is more comprehensive in multiple ways. First, the formulation of the family of reward-sensitive utility-discounted accuracies and the analysis of desirable properties and the computational complexity of finding BOPs are entirely new. Second, the empirical study is much more comprehensive and covers various aspects that have not been considered in the short version, especially in analyzing the potential (dis)advantages of the proposed ensemble learning framework in the presence of reward-sensitivity and noisy data. All the proofs of formal results stated in this paper (propositions and remarks) can be found in the appendices. 

\section{Probabilistic multi-class classification}\label{sec:Probabilistic multi-class classification}

This section recalls the basic setting of probabilistic MCC and discusses potential challenges. Along the way, we will introduce frequently used notations and terminologies. To facilitate the readability of this paper, some frequently used notations and acronyms are listed in the table
\ref{tab:Notation_and_acronyms} of Appendix \ref{app:Notation and acronyms}.

\subsection{Singleton prediction making}

In this section, we recall the standard setting where one wishes to predict for each instance a singleton prediction, i.e., a single class. 

\subsubsection{Reward-sensitive accuracies}

Let $\cX$ denote an instance space and $\cY =\{y^1, \ldots, y^K\}$ a finite set of classes. We assume that an instance $\bx \in \cX$ is (probabilistically) associated with members of $\cY$. Let $\mathbf{X}$ be a set of features. We denote by $\Prob(Y \given \bx)$ the conditional distribution of $Y$ given $\mathbf{X}=\bx$. Given training data $\cD = \{ (\bx_n,y_n)|n \in [N] := \{1, \ldots, N \}\}$ drawn independently from $\Prob(\mathbf{X}, Y)$, the goal in multi-class classification (MCC) is to learn a classifier $\bh$, which is a mapping $\cX \fromto \cY$ that assigns to each instance $\bx\in \cX$ a class $\hat{y} \defi \bh(\bx) \in \cY$. The quality of $\bh(\bx)$ is often assessed by using utilities $u:\cY \times \cY \fromto \mathbb{R}_{+}$, which are assumed, without loss of generality, to be the higher-the-better accuracies.     


Each utility $u:\cY \times \cY \fromto \mathbb{R}_{+}$ are often specified by a $K\times K$-matrix $\mathbf{R}_{i,j}$ where the element $r_{i,j} :=u(y^i, y^j)$ is the reward of predicting $y^i$ while the correct class is $y^j$. In the specific case of $0/1$ accuracy $u^{0/1}$, we have $u^{0/1}(y^i, y^j) = \llbracket i = j \rrbracket$\footnote{$\llbracket \cdot \rrbracket$ is the indicator function, i.e., $\llbracket A \rrbracket = 1$ if the predicate $A$ is true and $=0$ otherwise.}. When $u(y^i, y^i) \geq u(y^i, y^j)$, for any $i \in [K]$ and $i\neq j$, $u$ is often referred to as a reward-sensitive accuracy \citep{elkan2001foundations,lachiche2003improving,o2008cost}. Reward-sensitive accuracies might be more informative in applications where the direct costs, such as the costs associated with the decisions the end-user should spend to treat the predicted class, are well-estimated. 

For example, if one wishes to construct a decision-support system for Plant Leaf Diseases \citep{geetharamani2019identification} treatment, where an Unmanned Aerial Vehicle (UAV), which is equipped with a predictor $\bh$, visits different farms and calls for the visit of a domain expert when suspecting Leaf Diseases in a farm. Rather than employ one domain expert per Plant Leaf Disease, it might be more beneficial to employ one domain expert per Plant Type or group of Plant Types. Clearly, a wrong prediction might still be rewarded if it still calls for the visit of an appropriate domain expert. If the direct costs, such as the cost of sending the domain experts to the farms, are well-estimated, it might be leveraged to define reward-sensitive accuracies to better assess the quality of the predictor $\bh$.          

\subsubsection{Bayes-optimal predictions}

Since different end-users and/or different application domains may call for different utilities, it is desirable to construct predictors that allow optimization of multiple utilities. A sound way to tackle this problem is to train a probabilistic predictor $\bh: \cX \fromto \Delta^K$, where $\Delta^K$ is the probability simplex over $\cY$, which provides, for each query $\bx$, an estimate of $\Prob(\cY \given \bx)$, from which a BOP of any utility $u:\cY \times \cY \fromto \mathbb{R}_{+}$ can be defined: 
\begin{equation}\label{eq:BOP}
  \hat{y}_u \in \operatorname*{argmax}_{\bar{y} \in \cY} \mathbf{E}_{\Prob(\cY \given \bx)}\big[u(\bar{y},y) \big] = \operatorname*{argmax}_{\bar{y} \in \cY} \sum_{y \in \mathcal{Y}} u(\bar{y},y) \vec{p}(y \given \vec{x}) \,.
\end{equation}
A BOP $\hat{y}_u$ \eqref{eq:BOP} of any reward-sensitive accuracy $u$ can be found in $O(K^2)$ \citep{nguyen2025credal}[Appendix 4.1] by computing the expected scores $\mathbf{E}_{\Prob(\cY \given \bx)}\big[u(\bar{y},y) \big]$, $\bar{y} \in \cY$ and search for a prediction with the highest expected score. If $u$ is the $0/1$ accuracy $u^{0/1}$, we have $\mathbf{E}_{\Prob(\cY \given \bx)}\big[u^{0/1}(\bar{y},y) \big] = \vec{p}(\bar{y} \given \vec{x})$, $\bar{y} \in \cY$, and therefore a BOP $\hat{y}_{u^{0/1}}$ \eqref{eq:BOP} can be found easily in $O(K)$.  

Let $y$ be the true class of $\vec{x}$. To be accurate with respect to a given utility $u$, the predictor $\bh$ should provide an accurate estimate of $\Prob(\cY \given \bx)$ to ensure that $\hat{y}_u$ \eqref{eq:BOP}, i.e., the one with the highest expected score $\mathbf{E}_{\Prob(\cY \given \bx)}\big[u(\cdot,y) \big]$, is ideally $y$ or at least another class with high $u(y_u, y)$. This is, however, difficult to ensure in practice, often due to model inadequacy and/or data imperfection. Model inadequacy is encountered when the hypothesis space, from which a predictor is selected, is not expressive enough to encode the (probabilistic) relationship $\Prob(\cY \given \cX)$. Data imperfection can appear in various scenarios, for example, scarce or noisy training data and noisy test data. Noisy training data, such as images that contain random variation in brightness or color information, can be present during the data collection process. Noisy training data can also be introduced when enriching the training data using data augmentation techniques \citep{shorten2019survey} or synthetic data \citep{nguyen2025provably}. Noisy test images can also be encountered in practice. For example, in the Plant Leaf Disease example, the UAV might encounter noisy test images due to severe weather.  

\subsection{Set-valued prediction making}

Set-valued prediction making has been perceived as a sound approach to mitigate the potential risk of making incorrect predictions of the singleton prediction-making strategy. In the following, we discuss its two basic building blocks: Reward-sensitive set-utilities and set-valued BOPs. 

\subsubsection{Reward-sensitive set-utilities}

Set-valued prediction making assumes that the prediction can be any subset $Y$ of $\cY$. Let $\cY_{\mathrm{set}} := \{Y \given Y \subset \cY, Y \neq \emptyset\}$. Any utility $u:\cY \times \cY \fromto \mathbb{R}_{+}$ can be generalized to assess the quality of set-valued prediction as $U:\cY_{\mathrm{set}} \times \cY \fromto \mathbb{R}_{+}$, where $U(\bar{Y}, y) = u(\bar{y}, y)$ whenever $\bar{Y} = \{y\}$. Each set-utility $U:\cY_{\mathrm{set}} \times \cY \fromto \mathbb{R}_{+}$ are often specified by a $(2^K - 1)\times K$-matrix $\mathbf{R}_{\bar{Y},j}$ where the element $r_{\bar{Y},j} :=U(\bar{Y}, y^j)$ is the reward of predicting $\bar{Y}$ while the correct class is $y^j$. 

Often, set-utilities are designed to assess the trade-off between recall, i.e., the set-valued predictions should have a high chance of covering the correct class, and recall, i.e., the cardinality of the set-valued prediction should not be unnecessarily large \citep{mortier2021efficient,yang2016costs,zaffalon2012evaluating}. Moreover, set-utilities and their optimal set-valued predictions are often expected to satisfy desirable properties, such as the ones discussed in \citep{yang2016costs} and others explored in Section \ref{sec:Desirable properties}. 

Since $2^K - 1$ is exponential in $K$, specifying a $(2^K - 1)\times K$-matrix $\mathbf{R}_{\bar{Y},j}$ to define the corresponding set-utility $U$ can be costly. Notable efforts on set-valued prediction making reduce this burden by assuming that the rewards of predicting any non-singleton prediction $\bar{Y}$ can be specified using its cardinality $|\bar{Y}|$ and assigned to its members $\bar{y} \in \bar{Y}$. Representative examples of set-utilities of this type are utility-discounted accuracies \citep{mortier2021efficient,zaffalon2012evaluating}, which are direct generalizations of the $0/1$ accuracy $u^{0/1}$, where $u^{0/1}(y^i, y^j) = \llbracket i = j \rrbracket$. Commonly used utility-discounted accuracies \citep{mortier2021efficient,nguyen2018reliable,zaffalon2012evaluating} can be written as
\begin{equation}
    U^{0/1}(\bar{Y},y) = g(|\bar{Y}|) \llbracket y \in \bar{Y} \rrbracket \,,
\end{equation}
where $g(|\bar{Y}|)$ is a monotone decreasing function of $|\bar{Y}|$. 

A quite popular and easy-to-interpret family of utility-discounted accuracies is $U_{\alpha}$. It reads
\begin{equation}\label{eq:U-01-alpha}
    U^{0/1}_{\alpha}(\bar{Y},y) = \left(\alpha + \frac{1 - \alpha}{|\bar{Y}|}\right) \frac{1}{|\bar{Y}|} \llbracket y \in \bar{Y} \rrbracket \,,
\end{equation}
and has been receiving increasing attention. Besides having various desirable properties as originally studied in \citep{zaffalon2012evaluating}, being re-examined in related work (See, e.g., \citep{mortier2021efficient,yang2016costs}), and also in this paper in Section \ref{sec:Desirable properties}, finding set-valued BOPs of utility-discounted accuracies $U^{0/1}_{\alpha}$ can be done efficiently in $O(K\log(K))$ \citep{mortier2021efficient}---we shall recall this property in the next section. Unfortunately, despite a rather rich literature in the construction and optimization of reward-sensitive accuracies, we are unaware of utility-discounted accuracies generalizing other reward-sensitive accuracies with both desirable properties and efficient computation of set-valued~BOPs.    

\subsubsection{Set-valued Bayes-optimal predictions}\label{sec:Set-valued Bayes-optimal predictions}

As a direct generalization of the standard BOP $\hat{y}_u$ \eqref{eq:BOP}, a set-valued BOP of any set-utility $U:\cY_{\mathrm{set}} \times \cY \fromto \mathbb{R}_{+}$ can be defined as 
\begin{equation}\label{eq:set-valued BOP}
  \hat{Y}_U \in \operatorname*{argmax}_{\bar{Y} \in \cY_{\mathrm{set}}} \mathbf{E}_{\Prob(\cY \given \bx)}\big[U(\bar{Y},y) \big] = \operatorname*{argmax}_{\bar{Y} \in \cY_{\mathrm{set}}} \sum_{y \in \mathcal{Y}} U(\bar{Y},y) \vec{p}(y \given \vec{x}) \,.
\end{equation}
Without taking into account the concrete formula of $U$ to eliminate sub-optimal set-valued predictions, one would have to exhaustively compute and compare $2^K-1$ expected scores $\mathbf{E}_{\Prob(\cY \given \bx)}\big[U(\bar{Y},y) \big]$, one per possible prediction $\bar{Y} \in \cY_{\mathrm{set}}$. This will be an unaffordable computational cost in many applications.

In the case of utility-discounted accuracy $U^{0/1}_{\alpha}$, it has been shown by \citep{mortier2021efficient} that solving~\eqref{eq:set-valued BOP} can be done in $O(K\log(K))$ for a simple reason: if we order the classes $y^{(1)}, y^{(2)}, \ldots, y^{(K)}$ such that $p(y^{(i)}|x) \geq p(y^{(i+1)}|x)$, then the optimal subset $\hat{Y}_U$ will be of the form $\{y^{(1)},y^{(2)},\ldots,y^{(j)}\}$ for some $j$. Moreover, it guarantees that the recall of optimizing $U^{0/1}_{\alpha}$ is always at least the accuracy of optimizing the $0/1$ accuracy $u^{0/1}$, as $\hat{Y}_{U^{0/1}_{\alpha}}$ of any utility-discounted accuracy $U^{0/1}_{\alpha}$ always contains the class with the highest $\vec{p}(\bar{y} \given \vec{x})$.

The remainder of this paper is devoted to tackling two challenges of probabilistic MCC: seeking accurate estimates of $\Prob(\cY \given \bx)$, and designing utility-discounted accuracies, which generalize reward-sensitive accuracies, with both desirable properties and efficient computation of set-valued~BOPs. In particular, we will show that our formulation allows for a similar trick as the one used by \citep{mortier2021efficient} for $U^{0/1}_{\alpha}$.

\section{Compact cautious deep ensembling}\label{sec:Compact credal deep ensembling}

This section presents a compact cautious deep ensembling that seeks accurate estimates of $\Prob(\cY \given \bx)$ using compact ensemble learning techniques and reliable set-valued predictions via the construction of utility-discounted accuracies with both desirable properties and efficient computation of set-valued~BOPs.

\subsection{Generation}

Ensemble learning \citep{dietterich2000ensemble,nguyen2025credal,Zhou2021ensemblelearning} is a common practice in seeking accurate and robust probabilistic predictors $\Prob: \cX \fromto \Delta^K$, which is often important in ensuring satisfactory levels of predictive performance of the entire probabilistic MCC approach.  

Let $\mathbf{D} = \{(y_n,\bx_n) | n \in [N] \} \subset \cY \times \cX $ be a training dataset. Ensemble generation constructs an ensemble $\mathbf{H} \defi \{\bh^j \given j \in [J] \defi \{1, \ldots, J\}\}$ of $J$ probabilistic classifiers $\bh^j \in \mathcal{H}$, $j \in [J]$ that give, for each instance $\bx$, a set of probabilities denoted by
\begin{align}
   \mathbf{H}(\bx) \defi \left\{ \bh^j(\bx) \given j \in [J]\right\} = \left\{\Prob^j(\cY \given \bx) \defi \left(p^j_1, p^j_2, \ldots, p^j_K \right) \given j \in [J] \right\}\,,
\end{align}
where $p^j_k \defi \Prob^j(y^k \given \bx)$, $k \in [K]$.

Conventional ensemble learning techniques, such as bagging, boosting, stacking, and random initialization \citep{yang2023survey}, typically require one to train and maintain a reasonably large set of ensemble members $\mathbf{H}(\bx)$ to ensure diversity. This is depicted in Fig.~\ref{fig:credal_ensebling}. Clearly, the training and storage burdens are visible challenges when deploying these techniques.    

\def\layersep{1.8cm}
\newcommand\nn[2]{
    \foreach \y in {1,...,2}
        \node[neuron, fill=green!10] (i\y-#1) at (0,\y + 0.5) {\tiny$X^\y$};

    \foreach \y in {1,...,3}
        \path node[neuron, fill=blue!10] (h\y-#1) at (\layersep,\y) {\tiny$h\y$};

    \node[neuron, fill=red!10] (o-#1) at (2*\layersep,2) {\tiny $\Prob^{#2}(Y \given \bX)$};

    \foreach \source in {1,...,2}
        \foreach \dest in {1,...,3}
            \path (i\source-#1) edge (h\dest-#1);

    \foreach \source in {1,...,3}
        \path (h\source-#1) edge (o-#1);
}

\begin{figure}[http]
    \centering
\def\layersep{1.8cm}
     \centering
\begin{tikzpicture}[
    shorten >=1pt,->,draw=black!70, node distance=\layersep,
    neuron/.style={circle,fill=black!25,minimum size=10,inner sep=0},
    edge/.style 2 args={pos={(mod(#1+#2,2)+1)*0.33}, font=\tiny},
    distro/.style 2 args={
        edge={#1}{#2}, node contents={}, minimum size=0.6cm, path picture={\draw[double=orange!50,white,thick,double distance=1pt,shorten >=0pt] plot[variable=\t,domain=-1:1,samples=51] ({\t},{0.2*exp(-100*(\t-0.05*(#1-1))^2 - 3*\t*#2))});}
    },
    weight/.style 2 args={
        edge={#1}{#2}, node contents={\pgfmathparse{0.35*#1-#2*0.15}\pgfmathprintnumber[fixed]{\pgfmathresult}}, fill=white, inner sep=2pt
    }
]
    
    \node[rectangle,draw] at (2,2) {\tiny $\mathbf{D} = \left\{(y_n,\mathbf{x}_n)\right\}_{n=1}^N$ \tikzmark{median}};

    \begin{scope}[yshift=1.5cm, xshift=4cm]
        \nn{regular}{1}
    \end{scope}
    \foreach \i in {1,...,2}
        \foreach \j in {1,...,3}
            \path (i\i-regular) -- (h\j-regular) node[weight={\i}{\j}];
    \foreach \i in {1,...,3}
        \path (h\i-regular) -- (o-regular) node[weight={\i}{2}]; 

    \begin{scope}[xshift=4cm]
        \node[] at (1.8,2) {$\ldots$};
    \end{scope}

    \begin{scope}[yshift=-1.5cm, xshift=4cm]
        \nn{regular}{J}
    \end{scope}
    \foreach \i in {1,...,2}
        \foreach \j in {1,...,3}
            \path (i\i-regular) -- (h\j-regular) node[weight={\i}{\j-1}];
    \foreach \i in {1,...,3}
        \path (h\i-regular) -- (o-regular) node[weight={\i}{3}];

    \begin{scope}[xshift=8cm]
        \node[rectangle,draw] at (1.8,2) {\tiny Median $\Prob_d(Y \given \bX)$ \tikzmark{median}};
    \end{scope}
    
    \begin{scope}[yshift=1.5cm,xshift=10cm]
    \node[rectangle,draw] at (1.8,2) {\tikzmark{BOP_m}\tiny BOP $\hat{y}_u$};
    \end{scope}
    
    \begin{scope}[xshift=10cm]
    \node[rectangle,draw] at (1.8,2) {\tikzmark{BOP_dots}\ldots};
    \end{scope}
    
    \begin{scope}[yshift= -1.5cm,xshift=10cm]
    \node[rectangle,draw] at (1.8,2) {\tikzmark{BOP_M}\tiny BOP $\hat{Y}_U$};
    \end{scope}

    \draw[cyan!40,->] (2.2, 2.5) -- (4, 3.5);
    \draw[cyan!40,->] (3.7, 2) -- (5.5, 2);
    \draw[cyan!40,->] (2.2, 1.5) -- (4, 0.5);

    \draw[cyan!40,->] (8.5, 3.1) -- (9.5, 2.4);
    \draw[cyan!40,->] (6.8, 2) -- (8.5, 2);
    \draw[cyan!40,->] (8.5, 0.9) -- (9.5, 1.6);

    \draw[cyan!40,->] (10.5, 2.4) -- (11.5, 3.1);
    \draw[cyan!40,->] (11, 2) -- (11.35, 2);
    \draw[cyan!40,->] (10.5, 1.6) -- (11.5, 0.9);  
\end{tikzpicture}    
    \caption{Illustration of the conventional ensemble learning framework}
    \label{fig:credal_ensebling}
\end{figure}

To mitigate the training and storage burdens of the generation phase, we propose to employ compact deep ensembles, whose ensemble members are essentially generated in an instance-wise manner during the aggregation phase by distorting a pre-trained deep neural network, to define the representative probability distribution from which a BOP of any utility can be defined. Besides being less costly, we will show in the next section that if the distortion process can guarantee the diversity (to some extent), it is reasonable to expect promising levels of predictive performance from such compact deep ensembles.

Concrete examples are Convolutional Neural Networks (CNNs) with the Monte-Carlo dropout prediction option \citep{lemay2022improving} and Bayesian Neural Networks (BNNs) \citep{foong2020expressiveness,jospin2022hands}. As depicted in Fig.~\ref{fig:families_of_ensembles}, they can be interpreted as representatives of two families of ensembles with different natures. While BNNs can be seen as ensembles of DNNs with the same architecture (and different parameters), CNNs with Monte Carlo dropout can be seen as ensembles of different architectures (and overlapping parameters). The aggregation phase of BNNs can be interpreted as generating a set of predictors by distorting the predictor specified by the means of the weight distributions. The aggregation phase of CNNs with Monte Carlo dropout can be interpreted as generating a set of predictors by distorting the CNN learned from the training data. 

 \begin{figure}[http!]
     \centering
     \begin{subfigure}[b]{\textwidth}
     \centering
\begin{tikzpicture}[
    shorten >=1pt,->,draw=black!70, node distance=\layersep,
    neuron/.style={circle,fill=black!25,minimum size=10,inner sep=0},
    edge/.style 2 args={pos={(mod(#1+#2,2)+1)*0.33}, font=\tiny},
    distro/.style 2 args={
        edge={#1}{#2}, node contents={}, minimum size=0.6cm, path picture={\draw[double=orange,white,thick,double distance=1pt,shorten >=0pt] plot[variable=\t,domain=-1:1,samples=51] ({\t},{0.2*exp(-100*(\t-0.05*(#1-1))^2 - 3*\t*#2))});}
    },
    weight/.style 2 args={
        edge={#1}{#2}, node contents={\pgfmathparse{0.35*#1-#2*0.15}\pgfmathprintnumber[fixed]{\pgfmathresult}}, fill=white, inner sep=2pt
    }
]
    \nn{regular}{}

    \foreach \i in {1,...,2}
        \foreach \j in {1,...,3}
            \path (i\i-regular) -- (h\j-regular) node[weight={\i}{\j}];
    \foreach \i in {1,...,3}
        \path (h\i-regular) -- (o-regular) node[weight={\i}{1}];

    \begin{scope}[yshift=2.5cm]
        \node[rectangle,draw] at (1.5,1.8) {\tiny $\mathbf{D} = \left\{(y_n,\mathbf{x}_n)\right\}_{n=1}^N$ };
    \end{scope}
    
    \begin{scope}[yshift=1.5cm, xshift=4.5cm]
        \nn{regular}{1}
    \end{scope}
     \foreach \i in {1}
        \foreach \j in {1,3}
            \path (i\i-regular) -- (h\j-regular) node[weight={\i}{\j}];
    \foreach \i in {1,3}
        \path (h\i-regular) -- (o-regular) node[weight={\i}{1}];

    \begin{scope}[xshift=4.5cm]
        \node[] at (1.8,2) {$\ldots$};
    \end{scope}

    \begin{scope}[yshift=-1.5cm, xshift=4.5cm]
        \nn{regular}{J}
    \end{scope}
     \foreach \i in {1,2}
        \foreach \j in {2,3}
            \path (i\i-regular) -- (h\j-regular) node[weight={\i}{\j}];
    \foreach \i in {1,...,3}
        \path (h\i-regular) -- (o-regular) node[weight={\i}{1}];

    \begin{scope}[xshift=8cm]
        \node[rectangle,draw] at (1.8,2) {\tiny Median $\Prob_d(Y \given \bX)$ \tikzmark{median}};
    \end{scope}
    
    \begin{scope}[yshift=1.5cm,xshift=10cm]
    \node[rectangle,draw] at (1.8,2) {\tikzmark{BOP_m}\tiny BOP $\hat{y}_u$};
    \end{scope}
    
    \begin{scope}[xshift=10cm]
    \node[rectangle,draw] at (1.8,2) {\tikzmark{BOP_dots}\ldots};
    \end{scope}
    
    \begin{scope}[yshift= -1.5cm,xshift=10cm]
    \node[rectangle,draw] at (1.8,2) {\tikzmark{BOP_M}\tiny BOP $\hat{Y}_U$};
    \end{scope}

    \draw[cyan!40,->] (1.7, 3.85) -- (1.7, 3.3);

    \draw[cyan!40,->] (2.8, 3) -- (4, 3.5);
    \draw[cyan!40,->] (4.3, 2) -- (5.5, 2);
    \draw[cyan!40,->] (2.8, 1) -- (4, 0.5);

    \draw[cyan!40,->] (8.5, 3.1) -- (9.5, 2.4);
    \draw[cyan!40,->] (6.8, 2) -- (8.5, 2);
    \draw[cyan!40,->] (8.5, 0.9) -- (9.5, 1.6);

    \draw[cyan!40,->] (10.5, 2.4) -- (11.5, 3.1);
    \draw[cyan!40,->] (11, 2) -- (11.35, 2);
    \draw[cyan!40,->] (10.5, 1.6) -- (11.5, 0.9);    
\end{tikzpicture}
\caption{A CNN with the Monte-Carlo dropout prediction option. In each CNN, edges with no weight are dropped.}  \label{fig:CNN+Dropout}
\end{subfigure}

\begin{subfigure}[b]{\textwidth}
     \centering
\begin{tikzpicture}[
    shorten >=1pt,->,draw=black!70, node distance=\layersep,
    neuron/.style={circle,fill=black!25,minimum size=10,inner sep=0},
    edge/.style 2 args={pos={(mod(#1+#2,2)+1)*0.33}, font=\tiny},
    distro/.style 2 args={
        edge={#1}{#2}, node contents={}, minimum size=0.6cm, path picture={\draw[double=orange!50,white,thick,double distance=1pt,shorten >=0pt] plot[variable=\t,domain=-1:1,samples=51] ({\t},{0.2*exp(-100*(\t-0.05*(#1-1))^2 - 3*\t*#2))});}
    },
    weight/.style 2 args={
        edge={#1}{#2}, node contents={\pgfmathparse{0.35*#1-#2*0.15}\pgfmathprintnumber[fixed]{\pgfmathresult}}, fill=white, inner sep=2pt
    }
]
    \nn{bayes}{}
    
    \foreach \i in {1,...,2}
        \foreach \j in {1,...,3}
            \path (i\i-bayes) -- (h\j-bayes) node[distro={\i}{\j}];
    \foreach \i in {1,...,3}
        \path (h\i-bayes) -- (o-bayes) node[distro={\i}{1}];

    \begin{scope}[yshift=2.5cm]
        \node[rectangle,draw] at (1.5,1.8) {\tiny $\mathbf{D} = \left\{(y_n,\mathbf{x}_n)\right\}_{n=1}^N$ };
    \end{scope}

    \begin{scope}[yshift=1.5cm, xshift=4.5cm]
        \nn{regular}{1}
    \end{scope}
    \foreach \i in {1,...,2}
        \foreach \j in {1,...,3}
            \path (i\i-regular) -- (h\j-regular) node[weight={\i}{\j-1}];
    \foreach \i in {1,...,3}
        \path (h\i-regular) -- (o-regular) node[weight={\i}{1}]; 

    \begin{scope}[xshift=4.5cm]
        \node[] at (1.8,2) {$\ldots$};
    \end{scope}

    \begin{scope}[yshift=-1.5cm, xshift=4.5cm]
        \nn{regular}{J}
    \end{scope}
    \foreach \i in {1,...,2}
        \foreach \j in {1,...,3}
            \path (i\i-regular) -- (h\j-regular) node[weight={\i}{\j}];
    \foreach \i in {1,...,3}
        \path (h\i-regular) -- (o-regular) node[weight={\i}{3}];

    \begin{scope}[xshift=8cm]
        \node[rectangle,draw] at (1.8,2) {\tiny Median $\Prob_d(Y \given \bX)$ \tikzmark{median}};
    \end{scope}
    
    \begin{scope}[yshift=1.5cm,xshift=10cm]
    \node[rectangle,draw] at (1.8,2) {\tikzmark{BOP_m}\tiny BOP $\hat{y}_u$};
    \end{scope}
    
    \begin{scope}[xshift=10cm]
    \node[rectangle,draw] at (1.8,2) {\tikzmark{BOP_dots}\ldots};
    \end{scope}
    
    \begin{scope}[yshift= -1.5cm,xshift=10cm]
    \node[rectangle,draw] at (1.8,2) {\tikzmark{BOP_M}\tiny BOP $\hat{Y}_U$};
    \end{scope}

    \draw[cyan!40,->] (1.7, 3.85) -- (1.7, 3.3);
    
    \draw[cyan!40,->] (2.8, 3) -- (4, 3.5);
    \draw[cyan!40,->] (4.3, 2) -- (5.5, 2);
    \draw[cyan!40,->] (2.8, 1) -- (4, 0.5);

    \draw[cyan!40,->] (8.5, 3.1) -- (9.5, 2.4);
    \draw[cyan!40,->] (6.8, 2) -- (8.5, 2);
    \draw[cyan!40,->] (8.5, 0.9) -- (9.5, 1.6);

    \draw[cyan!40,->] (10.5, 2.4) -- (11.5, 3.1);
    \draw[cyan!40,->] (11, 2) -- (11.35, 2);
    \draw[cyan!40,->] (10.5, 1.6) -- (11.5, 0.9);    
\end{tikzpicture}
\caption{A BNN as an ensemble}  \label{fig:BNN}
\end{subfigure}
\caption{Illustration of compact ensembles.} \label{fig:families_of_ensembles}
\end{figure}

\subsection{Aggregation}

\subsubsection{Representative probability distributions}

Let $d: \Delta^K \times \Delta^K \fromto \mathbb{R}_{+}$ be any statistical distance. We follow the suggestion in \citep{nguyen2025credal} and formulate the problem of aggregating a given ensemble $\mathbf{H}(\bx)$ into a representative predictor $\bh_d$ as finding a predictor within the hypothesis space $\mathcal{H}$, which optimizes the average distance $d$ to the ensemble members. 
\begin{align}
   \bh_d &\in \operatorname*{argmin}_{\bh \in \mathcal{Y}^{\mathcal{X}}} \sum_{j =1}^J \mathbf{E}\left[ d(\bh, \bh^j)\right] = \operatorname*{argmin}_{\bh \in \mathcal{Y}^{\mathcal{X}}} \mathbf{E}\left[\sum_{j =1}^J d(\bh, \bh^j)\right]\label{eq:representative_classifier}\\
   &=  \operatorname*{argmin}_{\bh \in \mathcal{Y}^{\mathcal{X}}} \int_{\bx \in \cX}\left[\sum_{j =1}^J d(\bh(\bx), \bh^j(\bx))\right] d \bx \label{eq:representative_classifier_point} \,,
\end{align}
This notion of a representative predictor can be conceived as a Bayesian optimal inference, where $d$ plays the role of a loss function, and Equation~\ref{eq:representative_classifier} gives the representative probability minimizing this expected loss. 

The assumption $\bh_d \in \mathcal{Y}^{\mathcal{X}}$ means that $\bh_d$ is not restricted to any specific hypothesis space $\mathcal{H} \subsetneq \mathcal{Y}^{\mathcal{X}}$. Therefore, one minimizer $\bh_d$ of \eqref{eq:representative_classifier} can be defined in a pointwise manner, i.e., for each $\bx \in \cX$, $\bh_d$ can be any classifier that produces 
\begin{equation}\label{eq:representative distribution}
    \Prob_d(\cY \given \bx) = \operatorname*{argmin}_{\Prob(\cY \given \bx): \sum_{k=1}^K p_k =1} \sum_{j =1}^J d(\Prob(\cY \given \bx),\Prob^j(\cY \given \bx)) \,.
\end{equation}
As pointed out \citep{nguyen2023learning,nguyen2025credal} (and references therein), when $d$ is the squared Euclidean distance
\begin{equation}\label{eq:squared Euclidean distance}
   d_{\mathrm{sE}}(\Prob(\cY \given \bx), \Prob^j(\cY \given \bx)) = \sum_{k=1}^k |p_k-p^j_k|^2 \,,
\end{equation}
the representative distribution $\Prob_d(\cY \given \bx)$ \eqref{eq:representative distribution} is the average of $\Prob^j(\cY \given \bx)$, $j \in [J]$, i.e., $\bh_d$ is the graded majority voting ensemble. This also suggests that changing the distance $d$ in \eqref{eq:representative distribution} may lead to other graded majority voting ensembles. For other distances, such as $L_1$ distance
\begin{equation}\label{eq:L_1 distance}
   d_{\mathrm{L1}}(\Prob(\cY \given \bx), \Prob^j(\cY \given \bx)) = \sum_{k=1}^k |p_k-p^j_k| \,,
\end{equation}
and Kullback–Leibler (KL) divergence
\begin{equation}\label{eq:KL divergence}
   d_{\mathrm{KL}}(\Prob(\cY \given \bx), \Prob^j(\cY \given \bx)) = \sum_{k=1}^k p_k \log\left(\frac{p_k}{p^j_k}\right) \,,
\end{equation}
the problem of finding the representative distribution $\Prob_d(\cY \given \bx)$ \eqref{eq:representative distribution} may be less obvious, but can still be found rather efficiently by solving convex optimization problems as long as the distance $d$ is convex \citep{nguyen2025credal}. We refer to \citep{nguyen2025credal} for a detailed discussion on the computational aspects of finding the representative distribution $\Prob_d(\cY \given \bx)$~\eqref{eq:representative distribution}.

Within this formulation, it is clear that if, for each query $\bx$, one can ensure the accuracy and diversity of $\Prob^j(\cY \given \bx)$, $j\in [J]$, whether these probabilistic predictions are provided by a conventional ensemble or a compact ensemble might be less crucial.

Once the representative distribution $\Prob_d(\cY \given \bx)$ \eqref{eq:representative distribution} is found, it can be used to define the set-valued BOP \eqref{eq:set-valued BOP} of any reward-sensitive set-utility $U$ as discussed in Section \ref{sec:Set-valued Bayes-optimal predictions}. In the next section, we shall present reward-sensitive set-utilities that generalize both the commonly used utility-discounted accuracies $U^{0/1}_{\alpha}$ \eqref{eq:U-01-alpha} and reward-sensitive accuracies, and then show that their BOPs can be found efficiently. More importantly, we will show in Section \ref{sec:Desirable properties} that the proposed family of reward-sensitive set-utilities satisfies various desirable properties, including the ones discussed in \citep{yang2016costs}. 

\subsubsection{Reward-sensitive utility-discounted accuracies}

Inspired by the notion of utility-discounted accuracies $U_{\alpha}$ \eqref{eq:U-01-alpha}, in this paper, we propose reward-sensitive utility-discounted accuracies, where the rewards of predicting any non-singleton prediction $\bar{Y}$ can be specified using its cardinality $|\bar{Y}|$ and the reward-sensitive accuracies $u(\bar{y}, y)$ assigned to its members $\bar{y} \in \bar{Y}$. Therefore, the proposed notion of reward-sensitive utilities can be employed to generalize any given reward-sensitive accuracy without requiring additional cost to estimate rewards (of predicting non-singleton predictions). 

We first rewrite utility-discounted accuracies $U_{\alpha}$ \eqref{eq:U-01-alpha} as 
\begin{align}\label{eq:U-01-alpha-rewritten}
    U^{0/1}_{\alpha}(\bar{Y},y) &= \left(\alpha + \frac{1 - \alpha}{|\bar{Y}|}\right) \frac{1}{|\bar{Y}|} \llbracket y \in \bar{Y} \rrbracket  \\
    &= \left(\alpha + \frac{1 - \alpha}{|\bar{Y}|}\right) \frac{1}{|\bar{Y}|} \sum_{ \bar{y} \in \bar{Y}} \llbracket \bar{y} = y \rrbracket \\
    &= \left(\alpha + \frac{1 - \alpha}{|\bar{Y}|}\right) \frac{1}{|\bar{Y}|} \sum_{ \bar{y} \in \bar{Y}} u^{0/1}(\bar{y}, y)\,.
\end{align}
With this reformulation, it is clear that $U_{\alpha}$ quantifies the trade-off between recall and precision using the product of the total recall assigned to $\bar{Y}$
\begin{equation}
    \sum_{ \bar{y} \in \bar{Y}} u^{0/1}(\bar{y}, y)
\end{equation}
and the rational function of precision
\begin{equation}\label{eq:rational function of precision}
    g_{\alpha}(|\bar{Y}|) = \left(\alpha + \frac{1 - \alpha}{|\bar{Y}|}\right) \frac{1}{|\bar{Y}|} \,, 
\end{equation}
which is a monotone decreasing function of $|\bar{Y}|$. The hyperparameter $\alpha$ provides a simple way to control how quickly the reward for covering the true class decreases when increasing the cardinality $|\bar{Y}|$.  

Proceeding by analogous logic, for any given reward-sensitive accuracy $u$, we define a family of reward-sensitive utility-discounted accuracies $U$:
\begin{equation}\label{eq:Reward-sensitive-U-alpha}
    U_{\alpha}(\bar{Y},y) = \left(\alpha + \frac{1 - \alpha}{|\bar{Y}|}\right) \frac{1}{|\bar{Y}|} \sum_{ \bar{y} \in \bar{Y}} u(\bar{y}, y)\,,
\end{equation}
where 
\begin{equation}
    \sum_{ \bar{y} \in \bar{Y}} u(\bar{y}, y)
\end{equation}
is the total reward-sensitive recall assigned to $\bar{Y}$. 

Clearly, the reward-sensitive utility-discounted accuracy $U_{\alpha}$ \eqref{eq:Reward-sensitive-U-alpha} assigned to $\bar{Y}$ should be high if the total reward-sensitive recall is high and the cardinality $|\bar{Y}|$ is small. As will be shown in Section \ref{sec:Desirable properties}, we can determine, for each given reward-sensitive accuracy $u$, a range of $\alpha$ so that  $U_{\alpha}$ \eqref{eq:Reward-sensitive-U-alpha} satisfies various desirable properties. However, let us first show that finding set-valued BOPs \eqref{eq:set-valued BOP} of any $U_{\alpha}$ \eqref{eq:Reward-sensitive-U-alpha} can be done efficiently. 

\subsubsection{Set-valued Bayesian-optimal predictions}

We will proceed with a general setting where $U_{\alpha}$ are generalizations of a given reward-sensitive $u$, such that, for any true class $y \in \cY$, $1 = u(y,y) > u(\bar{y},y) \geq 0$, $\bar{y} \neq y \in \cY$. In other words, we only assume that the reward given for the correct singleton predictions should not be smaller than making wrong singleton predictions. Note that as long as $u$ is bounded, those constraints are not limiting, as the expectation operator is positive homogeneous, i.e., $\mathbf{E}_{\bp(\cY \given x)}\left[ \alpha f + \beta \right]=\alpha \mathbf{E}_{\bp(\cY \given x)}[f] + \beta$ for any $\alpha >0$ and $\beta \in \mathbb{R}$, hence one can always rescale $u$ between $[0,1]$ without modifying the optimal decisions and the solution of Equation~\eqref{eq:set-valued BOP}. 

\begin{proposition}\label{pro:set-valued BOP U_alpha}
Let $u$ be any reward-sensitive accuracy such that, for any true class $y \in \cY$, $1 = u(y,y) > u(\bar{y},y) \geq 0$, $\bar{y} \neq y \in \cY$. Let $\alpha \geq 1$. Let $U_{\alpha}$ \eqref{eq:Reward-sensitive-U-alpha} be the corresponding reward-sensitive utility-discounted accuracy. Finding a set-valued BOP \eqref{eq:set-valued BOP} of $U_{\alpha}$ can be done in $O(K^2)$ using Algorithm \ref{alg:set-valued BOP U_alpha}. 
\begin{algorithm}[http]
\DontPrintSemicolon
\KwIn{Probability distribution $\bp(\cY \given x)$, output space $\mathcal{Y}$, number of classes $K$, a reward-sensitive utility-discounted accuracy $U_{\alpha}$}
\KwOut{A set-valued BOP $\hat{Y}_{U_\alpha}$ \eqref{eq:set-valued BOP}}
\For{$k \gets 1$ \KwTo $K$}{
    Compute $s_k \gets \mathbf{E}_{\bp(\cY \given x)}\left[u(y^k,y)\right]$ \Comment{Cost $O(K)$}
}
Sort $s_1, \ldots, s_K$ such that $s_{(1)} \geq \ldots
   \geq s_{(K)}$  \Comment{Cost $O(K\log(K))$} 
   
$s_{\mathrm{temp}} \gets s_{(1)}$\;
$e_{1} \gets s_{(1)}$\;
\For{$k \gets 2$ \KwTo $K$}{
  Compute $g_{\alpha}(k)$ \Comment{Use \eqref{eq:rational function of precision} in $O(1)$}
  
  $s_{\mathrm{temp}} \gets s_{\mathrm{temp}} + s_{(k)}$
  
  $e_k \gets g_{\alpha}(k) s_{\mathrm{temp}}$ \Comment{Calculate $ \mathbf{E}_{\bp(\cY \given x)} \left[  U_{\alpha}( \{y^{(1)}, \ldots, y^{(k)}\},y) \right]$}
  }
Find $k^* = \operatorname*{argmax}_{k \in [K]}e_k$ \Comment{Find the optimal cardinality}

\Return{$\hat{Y}_{U_\alpha} = \{y^{(1)}, \ldots, y^{(k^*)}\}$}\;
\caption{Finding a set-valued BOP $\hat{Y}_{U_\alpha}$ \eqref{eq:set-valued BOP} of $U_{\alpha}$}
\label{alg:set-valued BOP U_alpha}
\end{algorithm}

\end{proposition}

While the full proof is in the appendix, a key element to understand Algorithm~\ref{alg:set-valued BOP U_alpha} is to observe that $\mathbf{E}_{\bp(\cY \given x)}\left[  U_{\alpha}(\bar{Y},y) \right]=g(|\bar{Y}|)\sum_{\bar{y} \in \bar{Y}} \mathbf{E}_{\bp(\cY \given x)}[u(\bar{y},y)]$. Later on, in Section \ref{sec:Desirable properties}, we shall show that reward-sensitive utility-discounted accuracies $U_{\alpha}$ satisfy the unimodal property: the expected scores $\mathbf{E}_{\bp(\cY \given x)} \left[  U_{\alpha}(\hat{Y}^k_{U_{\alpha}},y) \right]$, where $\hat{Y}^k_{U_{\alpha}}$ is local optimal set-valued prediction of fixed cardinality $k$, is a unimodal function of $k$. Therefore, we can stop the for-loop in lines 6--9 of Algorithm \ref{alg:set-valued BOP U_alpha} when $\mathbf{E}_{\bp(\cY \given x)} \left[  U_{\alpha}(\hat{Y}^k_{U_{\alpha}},y) \right]$ starts decreasing. This may help to significantly reduce the runtime of finding set-valued BOPs, as suggested by our empirical evidence that the average cardinality of non-singleton predictions is often small.     

\section{Desirable properties}\label{sec:Desirable properties}

This section discusses desirable properties that might be expected from utility-discounted accuracies, and then assesses the proposed reward-sensitive utility-discounted accuracies $U_{\alpha}$ \eqref{eq:Reward-sensitive-U-alpha} accordingly.  

\subsection{Properties of interest}

In the following, we reorganize desirable properties of reward-sensitive utility-discounted accuracies formulated in \citep{yang2016costs} for assessing individual set-utilities and introduce a new property for assessing families of set-utilities, namely nested coarsening. 

\begin{property}[Normalization] \label{proper:Normalization} A reward-sensitive utility-discounted accuracy $U:\cY_{\mathrm{set}} \times \cY \fromto \mathbb{R}_{+}$ is normalized if 
\begin{equation}
U(\bar{Y},y) \in [0,1], \forall (\bar{Y},y) \in \mathcal{Y}_{\mathrm{set}} \times \mathcal{Y} \,.    
\end{equation}
\end{property}
Normalization is indeed a desirable property, which facilitates the interpretation of the performance level of the set-valued predictors at both individual and sample/population levels.   

\begin{property}[Universal feasibility]\label{proper:Universal feasibility}
Let $\Delta^K$ be the probability simplex over $|\mathcal{Y}|$ outcomes. Let $\cY_{\mathrm{sin}}$ be the set of all the singleton predictions. A reward-sensitive utility-discounted accuracy $U:\cY_{\mathrm{set}} \times \cY \fromto \mathbb{R}_{+}$ is universal feasibility if there is a subset $\cY_{\mathrm{per}}$, which satisfies $\cY_{\mathrm{sin}} \subsetneq \cY_{\mathrm{per}} \subset \cY_{\mathrm{set}}$, and for any set-valued prediction $\hat{Y} \in \cY_{\mathrm{per}}$, there exists a conditional probability distribution $\boldsymbol{p}(\mathcal{Y}\given \bx) \in \Delta^K$ such that 
\begin{equation}
\hat{Y} \in \operatorname*{argmax}_{\bar{Y}\in \mathcal{Y}_{\mathrm{set}}} \mathbf{E}_{\boldsymbol{p}(\mathcal{Y}\given \bx)}\left[U(\bar{Y},y) \right] \,.    
\end{equation}
\end{property}

Universal feasibility ensures that both singleton and set-valued predictions have chances to be BOPs $\hat{Y}_U$ \eqref{eq:set-valued BOP} of the reward-sensitive utility-discounted accuracy $U$. Not having this property means that, in practice, the system does not allow for set-valued predictions, defeating the intent to have some. 

\begin{property}[Nested optimality]\label{proper:Nested optimality}
For any conditional probability distribution $\boldsymbol{p}(\mathcal{Y}\given \bx) \in \Delta^{|\mathcal{Y}|}$, and any possible cardinality $k \in [K]$, let 
\begin{align}
    \hat{Y}^{k}_U \in \operatorname*{argmax}_{\substack{\bar{Y} \in \mathcal{Y}_{\mathrm{set}} \\ |\bar{Y}| = k}} \mathbf{E}_{\boldsymbol{p}(\mathcal{Y}\given \bx)}\left[ U(\bar{Y}, y) \right]\,.
\end{align}
A reward-sensitive utility-discounted accuracy $U:\cY_{\mathrm{set}} \times \cY \fromto \mathbb{R}_{+}$ satisfies the nested optimality if 
\begin{equation}
\hat{Y}^{k}_U \subset \hat{Y}^{k'}_U, \forall 1 \leq k < k' \leq K\,.    
\end{equation}
\end{property}

By definition, nested optimality ensures that for any query instance $\bx$, the recall of $\hat{Y}^{k}_U$ should not decrease when allowing larger set-valued predictions. This might be clearly desirable for any set-valued prediction approach. It also encodes a kind of logical consistency: a previously selected class should not be removed if we allow for more classes to be included in the set-valued prediction.

To ensure that making BOPs $\hat{Y}_U$ \eqref{eq:set-valued BOP} of $U$ indeed seeks a trade-off between precision and recall, it might be desirable to require the following unimodality property.  

\begin{property}[Unimodality] \label{proper:Unimodality}
 A reward-sensitive utility-discounted accuracy $U:\cY_{\mathrm{set}} \times \cY \fromto \mathbb{R}_{+}$ satisfies unimodality if, for any conditional probability distribution $\boldsymbol{p}(\mathcal{Y}\given \bx) \in \Delta^K$, the local maximum expected score 
\begin{equation}
f^k_U =  \mathbf{E}_{\boldsymbol{p}(\mathcal{Y}\given \bx)}\left[U(\hat{Y}^{k}_U,y)\right]    
\end{equation} 
is a unimodal function of $k$. More precisely, there is a single peak $k^* \in [K]$ such that 
\begin{equation}
    f^k_U \leq \dots \leq f^{k^*}_U \land f^{k^*}_U \geq \dots \geq f^{K}_U \,.
\end{equation}
\end{property}
If nested optimality is already satisfied, this simply means that by making BOPs $\hat{Y}_U$ \eqref{eq:set-valued BOP} of $U$, we should stop seeking more recall (by increasing the cardinality of the set-valued prediction) as soon as the precision quickly drops.  

Since different end-users may prefer different ways to assess the trade-off between recall and precision, it might be desirable to facilitate their participation in the design of the set-utilities. One way to accommodate this practical need is to provide end-users with a family of reward-sensitive utility-discounted accuracies $U_\alpha$, $\alpha \in [\underline{\alpha}, \overline{\alpha}]$, whose members only differ by a real-valued hyperparameter $\alpha$ controlling the level of reward when the set-valued predictions become larger
\begin{align}
    U_{\alpha'}(\bar{Y},y) \geq U_{\alpha}(\bar{Y},y) \,, \forall \alpha' \geq \alpha\,, \forall (\bar{Y},y) \in \cY_{\mathrm{set}} \times \cY \,,
\end{align}    
and let them analyze the effect of hyperparameter $\alpha$ (e.g., using a validation set) and choose the final values at their convenience. 

In this case, it might be desirable to ensure the following logical consistency between the predictions when we increase our degree of robustness by controlling the reward level.

\begin{property}[Nested coarsening] \label{proper:Nested coarsening}
Let $U_\alpha$, $\alpha \in [\underline{\alpha}, \overline{\alpha}]$, be any family of reward-sensitive utility-discounted accuracies, whose members only differ by a real-valued hyperparameter $\alpha$ controlling the level of reward when the reward level becomes larger
\begin{align}
    U_{\alpha'}(\bar{Y},y) \geq U_{\alpha}(\bar{Y},y) \,, \forall \underline{\alpha} \leq \alpha' < \alpha \leq \overline{\alpha}\,, \forall (\bar{Y},y) \in \cY_{\mathrm{set}} \times \cY \,. 
\end{align} 

The family of $U_\alpha$, $\alpha \in [\underline{\alpha}, \overline{\alpha}]$, satisfies nested coarsening if 
\begin{align}\label{eq:set-inclusion}
    \hat{Y}_{U_{\alpha'}} \supseteq \hat{Y}_{U_{\alpha}} \,, \forall \underline{\alpha} \leq \alpha' < \alpha \leq \overline{\alpha}\,, \forall \boldsymbol{p}(\mathcal{Y}\given \bx) \in \Delta^K \,. 
\end{align}
\end{property}
Clearly, if the nested coarsening is ensured, for any query instance $\bx$, the recall (precision) should not decrease (increase) when increasing the reward level $\alpha$. As in Property~\ref{proper:Nested optimality}, the idea is that by allowing more imprecision in the prediction, a previously selected class should not be removed.   
\subsection{Analysis}

Reminding that the proposed reward-sensitive utility-discounted accuracies $U_{\alpha}$ \eqref{eq:Reward-sensitive-U-alpha} can be defined based on any given reward-sensitive $u$. In this section, we only focus on the cases where, for any true class $y \in \cY$, $1 \geq u(y,y) \geq u(\bar{y},y)\geq 0$, $\bar{y} \neq y \in \cY$. 

In the following, we show that for any given reward-sensitive accuracy $u$, there is a non-empty range $[1, \overline{\alpha}_u]$ so that, for any $\alpha \in (1, \overline{\alpha}_u)$ the reward-sensitive utility-discounted accuracy $U_{\alpha}$ \eqref{eq:Reward-sensitive-U-alpha} satisfies all the normalization, universal feasibility, nested optimality, and unimodality properties. Moreover, the value $\overline{\alpha}_u$ can be determined efficiently. We further show that the family of $U_{\alpha}$ \eqref{eq:Reward-sensitive-U-alpha}, $\alpha \in (1, \overline{\alpha}_u)$, satisfies the nested coarsening property.

\begin{proposition}\label{propos:desirable_properties_for_u_alpha}
Let $u$ be any reward-sensitive accuracy $u$, where, for any true class $y \in \cY$, $1 = u(y,y) > u(\bar{y},y) \geq 0$, $\bar{y} \neq y \in \cY$. Let $\alpha \in (1, \overline{\alpha}_u)$, where $\overline{\alpha}_u$ can be found in $O(K^3)$ using Algorithm \ref{alg:find max alpha}. The reward-sensitive utility-discounted accuracy $U_{\alpha}$ \eqref{eq:Reward-sensitive-U-alpha} satisfies all the normalization, universal feasibility, nested optimality, and unimodality properties.  

\begin{algorithm}[http]
\DontPrintSemicolon
 \KwIn{A reward-sensitivity accuracy $u$ and output space $\cY$}
\KwOut{The upper value $\overline{\alpha}_u$}
\For{$y \in \cY$}{
    Sort $y^1, \ldots, y^K$ s.t. $u(y^{(1),y},y) \geq \ldots
   \geq u(y^{(K),y},y)$  \Comment{Cost $O(K\log(K))$}
}

$\overline{\alpha}_u \gets \infty$

\For{$k \gets 2$ \KwTo $K$}{
    $s_k \gets \infty$\;
    \For{$y \in \cY$}{
        $s_{\mathrm{temp}} = 1 /(\sum_{i=1}^k u(y^{(i),y},y))$  \Comment{Cost $O(K)$}
        
        \If{$s_{\mathrm{temp}} < s_k$}{
            $s_k \gets s_{\mathrm{temp}}$\;
        }            
    }
    $\alpha_k = (s_k \times k^2 - 1)/(k - 1)$\;
    \If{$\alpha_k < \overline{\alpha}_u$}{
            $\overline{\alpha}_u \gets \alpha_k$\;
    }
} 
\Return{$\overline{\alpha}_u$}\;
\caption{Find the upper value $\overline{\alpha}_u$}
\label{alg:find max alpha}
\end{algorithm}
\end{proposition}

Since $U_{\alpha}$ satisfies the unimodality property, finding a set-valued BOP $\hat{Y}_{U_\alpha}$ \eqref{eq:set-valued BOP} of $U_{\alpha}$ can be further accelerated. Instead of assessing all the local optimal predictions $\hat{Y}^k_{U_\alpha}$ and then looking for the best one, one can gradually increase the cardinality $k$ and stop the searching algorithm as soon as the local maximum expected score $\mathbf{E}_{\boldsymbol{p}(\mathcal{Y}\given \bx)}\left[U_{\alpha}(\hat{Y}^{k}_{U_{\alpha}},y)\right]$ starts decreasing. This can be done using Algorithm \ref{alg:set-valued BOP U_alpha revised}. This result generalizes the known result for the specific case of utility-discounted accuracies $U^{0/1}_{\alpha}$ \citep{mortier2021efficient}.

\begin{algorithm}[http!]
\DontPrintSemicolon
\KwIn{Probability distribution $\bp(\cY \given x)$, output space $\mathcal{Y}$, number of classes $K$, a reward-sensitive utility-discounted accuracy $U_{\alpha}$}
\KwOut{A set-valued BOP $\hat{Y}_{U_\alpha}$ \eqref{eq:set-valued BOP}}
\For{$k \gets 1$ \KwTo $K$}{
    Compute $s_k \gets \mathbf{E}_{\bp(\cY \given x)}\left[u(y^k,y)\right]$ \Comment{Cost $O(K)$}
}
Sort $s_1, \ldots, s_K$ such that $s_{(1)} \geq \ldots
   \geq s_{(K)}$  \Comment{Cost $O(K\log(K))$} 
   
$s_{\mathrm{temp}} \gets s_{(1)}$; $e_{\mathrm{max}} \gets s_{(1)}$\;

$e_{\mathrm{temp}} \gets s_{(1)}$;$\hat{Y}_{U_\alpha} \gets \{y^{(1)}\}$\;
\For{$k \gets 2$ \KwTo $K$}{
  Compute $g_{\alpha}(k)$ \Comment{Use \eqref{eq:rational function of precision} in $O(1)$}
  
  $s_{\mathrm{temp}} \gets s_{\mathrm{temp}} + s_{(k)}$
  
  $e_{\mathrm{temp}} \gets g_{\alpha}(k) s_{\mathrm{temp}}$ \Comment{Calculate $ \mathbf{E}_{\bp(\cY \given x)} \left[  U_{\alpha}(\hat{Y}_{U_\alpha} \cup \{y^{(k)}\},y) \right]$}
  
  \eIf{$e_{\mathrm{temp}} < e_{\mathrm{max}}$}{
    Return $\hat{Y}_{U_\alpha}$ \Comment{Stop when the expected score starts decreasing}
  }{
    $e_{\mathrm{max}} \gets e_{\mathrm{temp}}$ \Comment{Update the max expected scores}

    $\hat{Y}_{U_\alpha} \gets \hat{Y}_{U_\alpha} \cup \{y^{(k)}\}$ \Comment{Update the optimal prediction}
  }
}
\Return{$\hat{Y}_{U_\alpha}$}\;
\caption{Finding a set-valued BOP $\hat{Y}_{U_\alpha}$ \eqref{eq:set-valued BOP} of $U_{\alpha}$}
\label{alg:set-valued BOP U_alpha revised}
\end{algorithm}

\begin{proposition}\label{propos:Nested coarsening}
Let $u$ be any reward-sensitive accuracy $u$, where, for any true class $y \in \cY$, $1 = u(y,y) > u(\bar{y},y) \geq 0$, $\bar{y} \neq y \in \cY$. The family of reward-sensitive utility-discounted accuracies $U_{\alpha}$ \eqref{eq:Reward-sensitive-U-alpha}, $\alpha \in (1, \overline{\alpha}_u)$, satisfies the nested coarsening property.  
\end{proposition}

We end our analysis with the remarks for the specific case of utility-discounted accuracies $U^{0/1}_{\alpha}$ \citep{mortier2021efficient}. 

\begin{remark}\label{rem:upper-value-for-U-0/1}
For any $\alpha \in (1, 3)$, the utility-discounted accuracy $U^{0/1}_{\alpha}$ satisfies all the normalization, universal feasibility, nested optimality, and unimodality properties. Moreover, The family of utility-discounted accuracies $U^{0/1}_{\alpha}$ \eqref{eq:Reward-sensitive-U-alpha}, $\alpha \in (1, 3)$, satisfies the nested coarsening property.   
\end{remark}

\begin{remark}\label{rem:probustness-permissiveness-for-U-0/1}
For any utility-discounted accuracy $U^{0/1}_{\alpha}$ with $\alpha \in (1, 3)$, we have the set $\cY_{\mathrm{per}} = \cY_{\mathrm{set}}$.  
\end{remark}

\section{Experiments}\label{sec:Experiments}

\subsection{Experimental setting}

In the following, we describe the data sets and evaluation metrics. We then describe $8$ sets of experiments designed based on $8$ different combinations of distribution shift scenario and reward-sensitivity. Finally, we provide a detailed description of the two compact ensembles used to estimate the representative distribution \eqref{eq:representative distribution} in each experiment. 

The source code used in this paper is developed on top of the Bayesian-Torch library \href{https://github.com/IntelLabs/bayesian-torch}{https://github.com/IntelLabs/bayesian-torch} and has been made public at \href{https://github.com/kdimtensor/Compact-cautious-deep-ensembling-in-multi-class-classification}{https://github.com/kdimtensor/Compact-cautious-deep-ensembling-in-multi-class-classification}. 

\subsubsection{Data sets and evaluation metrics}

We conduct experiments on $3$ image data sets, which are summarized in Table \ref{tab:datasets}.

\begin{table}[!http]
\centering
\begin{tabular}{l| c | c c c}
\hline
\# &Name  &\# of instances & image type &\# of classes \\
\hline
1 &CIFAR-10 &60,000    &$32 \times 32$ color images  & 10  \\
2 &Fashion-MNIST &70,000    &$28 \times 28$ gray images  & 10  \\
3 &Plant Leaf Diseases &61,486    &$256 \times 256$ color images  & 39  \\
\hline
\end{tabular}
\caption{Data sets.}\label{tab:datasets} 
\end{table}
The Plant Leaf Diseases data set has been introduced in \citep{geetharamani2019identification} and has been made public at \href{https://data.mendeley.com/datasets/tywbtsjrjv/1}{https://data.mendeley.com/datasets/tywbtsjrjv/1}.

We construct three classifiers, namely clf\_sE, clf\_L1, and clf\_KL, which respectively produce the representative distributions optimizing the average $d_{\mathrm{sE}}$ \eqref{eq:squared Euclidean distance}, $d_{\mathrm{L1}}$ \eqref{eq:L_1 distance}, and $d_{\mathrm{KL}}$ \eqref{eq:KL divergence} distances. Once the representative distribution is given, a set-valued BOP $\hat{Y}_{U_\alpha}$ \eqref{eq:set-valued BOP}, which optimizes the expected reward-sensitive utility-discounted accuracy $U_\alpha(\hat{Y},y)$ \eqref{eq:Reward-sensitive-U-alpha} is predicted. The BOP $\hat{Y}_{U_\alpha}$ \eqref{eq:set-valued BOP} is computed using Algorithm \ref{alg:set-valued BOP U_alpha revised}. In the specific case of $\alpha = 1$, clf\_sE, clf\_L1, and clf\_KL are reduced to conventional classifiers, which always predict singletons \eqref{eq:BOP}.     

\subsubsection*{Assess the set-valued predictions}

We shall employ a few related metrics to assess whether these classifiers, which have been shown to have desirable properties, can bring a good trade-off between correctness and precision to mitigate the potential unrobustness of the probabilistic classifier. We sample $20$ values $\alpha \in [1, \overline{\alpha}_u]$, where $\overline{\alpha}_u$ is computed using Algorithm \ref{alg:find max alpha}, with equal step sizes and report the average scores of the following metrics on the test set $\mathbf{T} = \{(\vec{x}_t, y_t) \given t = 1, \ldots, T\}$:
\begin{itemize}
    \item $\mathrm{ual}$: The average $U_{\alpha}(\hat{Y}_{U_\alpha}, y)$ \eqref{eq:Reward-sensitive-U-alpha} 
    \begin{equation}\label{eq:test_ual}
        \mathrm{ual}(\mathbf{T}, \alpha) = \frac{1}{T} \sum_{t=1}^T U_{\alpha}(\hat{Y}^t_{U_\alpha}, y_t) \,. 
    \end{equation}
    \item $\mathrm{rec}$ The average recall
    \begin{equation}\label{eq:test_rec}
        \mathrm{rec}(\mathbf{T}, \alpha) = \frac{1}{T} \sum_{t=1}^T \max_{\bar{y} \in \hat{Y}^t_{U_\alpha}} u(\bar{y},y_t)   \,.
    \end{equation} 
    \item $\mathrm{pro}$: The proportion of non-singleton predictions
    \begin{equation}\label{eq:test_pro}
        \mathrm{pro}(\mathbf{T}, \alpha) = \frac{1}{T} \sum_{t=1}^T \llbracket |\hat{Y}^t_{U_\alpha}| > 1  \rrbracket \,.
    \end{equation}
    \item $\mathrm{siz}$: The average set-size $|\hat{Y}^t_{U_\alpha}|$ of non-singleton predictions
    \begin{equation}\label{eq:test_siz}
        \mathrm{siz}(\mathbf{T}, \alpha) = \frac{\sum_{t=1}^T  \llbracket |\hat{Y}^t_{U_\alpha}| > 1  \rrbracket |\hat{Y}^t_{U_\alpha}|}{\sum_{t=1}^T \llbracket |\hat{Y}^t_{U_\alpha}| > 1  \rrbracket} \,.
    \end{equation}
    \item $\mathrm{sin}$: The proportion of correct singleton predictions
    \begin{equation}\label{eq:test_sin}
        \mathrm{sin}(\mathbf{T}, \alpha) =  \frac{\sum_{t=1}^T \llbracket |\hat{Y}^t_{U_\alpha}| = 1  \rrbracket \llbracket y_t \in \hat{Y}^t_{U_\alpha}  \rrbracket }{\sum_{t=1}^T \llbracket |\hat{Y}^t_{U_\alpha}| = 1  \rrbracket}\,,
    \end{equation}
    \item $\mathrm{set}$: The proportion of correct set-valued predictions
    \begin{equation}\label{eq:test_set}
        \mathrm{set}(\mathbf{T}, \alpha) =  \frac{\sum_{t=1}^T \llbracket |\hat{Y}^t_{U_\alpha}| > 1  \rrbracket \llbracket y_t \in \hat{Y}^t_{U_\alpha} \rrbracket }{\sum_{t=1}^T \llbracket |\hat{Y}^t_{U_\alpha}| > 1  \rrbracket}\,,
    \end{equation}
\end{itemize}

For each data set, we follow a 3-fold cross-validation procedure and report the average scores of the metrics covered.

To assess the potential risk of distribution shifts, which may be due to the presence of either noisy training images or noisy test images, for each data set, we conduct $4$ experiments: 
\begin{itemize}
    \item \textbf{Experiment 1}: Clean training images and clean test images
    \item \textbf{Experiment 2}: Clean training images noisy test images
    \item \textbf{Experiment 3}: Noisy training images and clean test images
    \item \textbf{Experiment 4}: Noisy training images and noisy test images
\end{itemize} 
We generate noisy images by injecting Gaussian noise into the original images. The choice of Gaussian noise was motivated by its popularity in real-world applications. First, Gaussian noise is often added to data as part of data augmentation to improve model robustness by exposing the model to varied and noisy inputs. Second, in various applications, Gaussian noise in digital images can arise during acquisition, e.g., sensor noise caused by poor illumination and/or high temperature, and/or transmission, e.g., electronic circuit noise. We employed the Gaussian noise injection mechanism proposed by \citet{hendrycks2019robustness} with the severity level 5 (the highest).

Yet, one might, of course, think of various other noise injection mechanisms. We think our choice of noise injection mechanism serves well to illustrate the potential applications of our generic compact cautious deep learning framework, which would be directly improved by taking into account current and emerging developments of compact ensembles and learning representative probability distributions.    

\subsubsection{Reward-sensitive utility-discounted accuracies}

Inspired by related work on designing hierarchical utilities (see, e.g., \citep{urbani2024harnessing} and references therein), we propose to construct the reward-sensitive utilities $u$ based on the hierarchical representation of classes. More precisely, we first partition the classes into different groups, and assign to the correct predictions, wrong predictions within the same group, and predictions across groups, a reward of $1$, $0.25$, or $0$, respectively. 

In general, we rely on WordNet \citep{wordnet} for the hierarchy of objects that are commonly used in the literature for the CIFAR-10 data set \citep{cifar_hierarchy} and the Fashion-MNIST data set \citep{fmnist_hierarchy}. Regarding the CIFAR-10 data set, we also consider the specifications about the superclass and classes in the official website (\url{https://www.cs.toronto.edu/\~kriz/cifar.html}) and divide the classes into two groups: animal and vehicle; similarly to Fashion-MNIST, we have three groups: clothing, footwear, and accessory. In the case of the Plant Leaf Diseases data set, we partition the classes into 15 groups accordingly to the plants species, one per plant: apple, blueberry, cherry, corn, grape, orange, peach, pepper, potato, raspberry, soybean, squash, strawberry, tomato, and an outlier. We refer to \href{https://data.mendeley.com/datasets/tywbtsjrjv/1}{https://data.mendeley.com/datasets/tywbtsjrjv/1} for the $39$ class names of this data~set. 

We also conduct experiments with the conventional utility-discounted accuracies $U^{0/1}_\alpha$ \eqref{eq:U-01-alpha} to complement known results on these set-utilities \citep{mortier2021efficient}. 

Again, we reiterate that our proposal can accommodate a wide range of reward-sensitive utilities. Therefore, the end-user can flexibly customize the reward-sensitive utilities to better reflect the way (set-valued) prediction should be assessed in their application domains.  

\subsubsection{Compact ensembles}

We shall end the experimental setting with descriptions of the two compact ensembles used to estimate the representative distribution \eqref{eq:representative distribution}.

\begin{itemize}
    \item \textbf{BNNs:} For each train-test split, we train a BNN, where all weights follow Gaussian distributions, on training data using the default setting recommended by the Bayesian-Torch library \href{https://github.com/IntelLabs/bayesian-torch}{https://github.com/IntelLabs/bayesian-torch}. For each test image, we feed the image to the BNN $100$ times to estimate $100$ probabilistic predictions, which are then used to estimate the representative distribution \eqref{eq:representative distribution}.  
    \item \textbf{CNNs with the Monte-Carlo dropout prediction option}: For each train-test split, we train a CNN on training data using the default setting recommended by the Bayesian-Torch library \href{https://github.com/IntelLabs/bayesian-torch}{https://github.com/IntelLabs/bayesian-torch}. For each test image, we perform Monte-Carlo dropout to estimate $100$ probabilistic predictions, which are then used to estimate the representative distribution \eqref{eq:representative distribution}. To facilitate the diversity of the ensemble members, we adopt the dropout rate of $0.2$. 
\end{itemize}

\subsection{Results}\label{sec:Results}

Since the trends in the improvements in predictive performance provided by making set-valued predictions are consistent across the experiments, we only elaborate on the results obtained when BNNs are employed to estimate the representative distributions \eqref{eq:representative distribution}, which are then used to define set-valued BOPs \eqref{eq:set-valued BOP} of the reward-sensitive utility-discounted accuracies \eqref{eq:Reward-sensitive-U-alpha} and defer other results to Appendix  \ref{sec:Additional experimental results}.

\input{tikz/bnn_train_clean}

\input{tikz/bnn_train_gnoise}

The results for clean training data and noisy data, which are given in the Figs.~\ref{fig:BNNs with clean training data} and \ref{fig:BNNs with noisy training data}, respectively, suggest a few interesting trends. 

\subsubsection{BNNs with clean training data}

The results given in Fig.~\ref{fig:BNNs with clean training data} on clean training data show multiple things, which we summarise below:
\begin{itemize}
    \item The first remark is that with clean training data, having clean or noisy test data makes a big difference, as accuracy significantly drops on test data, with a drop of almost 60\% on CIFAR-10 and Plant Leaf when considering only singleton predictions ($\alpha=1$). The use of set-valued predictions reduces this gap, with an affordable price in terms of precision, the average size remaining around 2;
    \item In all cases, set-valued prediction made by our framework improves results by trading recall for precision, as $\mathrm{ual}$ increases for all data sets. Of course, the gain is more significant when classifiers struggle at the start, as there is little gain to be made from set-valued prediction in case initial, precise classifiers are already quite good;
    \item there is no significant difference between the use of different distances, as the curves for $d_{\mathrm{sE}}$ \eqref{eq:squared Euclidean distance}, the $L_1$ distance $d_{\mathrm{L1}}$ \eqref{eq:L_1 distance} and KL divergence $d_{\mathrm{KL}}$ \eqref{eq:KL divergence} show the same trend. On CIFAR-10 and Plant Leaf, $d_{\mathrm{sE}}$ seems to produce more set-valued predictions, compensated by a gain in recall, as indicated by the ual curves that are almost identical.   
\end{itemize}

\subsubsection{BNNs with noisy training data}

The results given in Fig.~\ref{fig:BNNs with noisy training data} when training data is noisy again show interesting results:
\begin{itemize}
    \item First, while there is indeed a small drop in performances, due to the lesser quality of instances, yet we can see that the results are almost similar for both noisy and clean test data. This is due to the fact that clean test data (the "shifted" distribution in this scenario) are not likely to be out-of-distribution when the initial training is done on a noisy distribution, while noisy test data are likely to be OOD when the training set is clean. Note that for the Plan leaf data, the y-axis gap is very small compared than in Fig.~\ref{fig:BNNs with clean training data}, so there is no big differences;
    \item Second, allowing for set-valued prediction is quite beneficial here, again with a relatively limited number of classes when returning set-valued predictions. 
\end{itemize}
Other remarks made for the clean training data remain valid for this scenario, notably the little differences observed in taking different distances. 

Overall, we can make some first conclusions before those of Section~\ref{sec:Conclusion}: the presented framework does allow improvement in recall and performances when allowing for set-valued predictions, and this is particularly true in noisy scenarios (either in the training or test data). It is robust, working equally well on different ensemble regimes as well as different distances to compute representative probabilities. However, its capacity to solve issues in data distribution appears limited when the shift between the training and test set is significant and produces many OOD samples: in this case, it seems that encoding robustness and set-valued approaches in utilities and in the decision steps has some positive impacts, but those remain limited. In such cases, it would probably be desirable to also consider approaches relying on ensembles and sets of probabilities~\citep{wang2026learning,nguyen2025credal}

\subsubsection{Other results}

Similar results when the BNNs are employed to estimate the representative distributions \eqref{eq:representative distribution}, which are then used to define set-valued BOPs \eqref{eq:set-valued BOP} of the utility-discounted accuracies \eqref{eq:U-01-alpha} are given in Figs.~\ref{fig:BNNs with clean training data no sensitivity}--\ref{fig:BNNs with noisy training data no sensitivity} of Appendix  \ref{sec:Additional experimental results}.

The results given in Figs.~\ref{fig:CNNs with clean training data}--\ref{fig:CNNs with noisy training data no sensitivity} of Appendix  \ref{sec:Additional experimental results} suggest that similar trends on the mentioned metrics are observed when the CNNs with the dropout prediction option are employed to estimate the representative distributions \eqref{eq:representative distribution}, which are then used to define set-valued BOPs \eqref{eq:set-valued BOP} of the reward-sensitive utility-discounted accuracies \eqref{eq:Reward-sensitive-U-alpha} and set-valued BOPs \eqref{eq:set-valued BOP} of the utility-discounted accuracies \eqref{eq:U-01-alpha}.

It is indeed interesting to see that while CNNs provide competitive predictive performance when the training and testing data are of the same types, compared to those that are provided by BNNs, CNNs appear to be (much) less promising when encountering mismatches between training and test data types. In the next subsection, we further elaborate on the quality of the representative distributions \eqref{eq:representative distribution} provided by these compact ensembles to seek insights about this behavior.  
 
In the following, we illustrate the behaviors of the BNNs and CNNs when being employed as compact ensembles on a few test instances extracted from the CIFAR-10 data set. Fig.~\ref{fig:set-valued BOPs of BNNs vs CNNs} reports the probabilistic predictions 
\begin{align*}
    \Prob(\cY \given \bx )=\large(&\Prob(\mathrm{airplane}\given \bx),\Prob(\mathrm{automobile}\given \bx), \Prob(\mathrm{bird}\given \bx), \Prob(\mathrm{cat}\given \bx), \Prob(\mathrm{deer}\given \bx),\\
    &\Prob(\mathrm{dog}\given \bx), \Prob(\mathrm{frog}\given \bx), \Prob(\mathrm{horse}\given \bx), \Prob(\mathrm{ship}\given \bx), \Prob(\mathrm{truck} \given \bx)\large)
\end{align*}
and set-valued BOPs $\hat{Y}_{U_{1.63}}$ \eqref{eq:set-valued BOP} for two cat images. Cat is the most difficult class in this data set and is usually misclassified as a frog. On the clean version of image 1, both BNNs and CNNs provide set-valued BOPs of cardinality of either $2$ or $3$ which cover the correct class. However, on the noisy version, only BNNs and CNNs, which are trained on noisy training data, provide set-valued predictions of cardinality $2$ that cover the correct class. On both the clean and noisy versions of image $2$, only BNNs, which are trained on noisy training data, provide set-valued BOPs of cardinality $2$ that cover the correct class.  

\begin{figure}[http!]
    \centering
    \begin{tabular}{c|l|l}
        \hline
        \multicolumn{3}{c}{Image 1: Clean version}\\
        \hline
        \multirow{8}{*}{\includegraphics[width=0.22\textwidth]{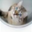}}  & \multirow{4}{*}{Clean $\mathbf{D}$} &  BNN: $\hat{Y}_{U_{1.63}} = \{\mathrm{bird}, \mathrm{cat}, \mathrm{frog}\}$\\
        && $\Prob: (0.00,   0.00,   0.38, 0.25, 0.1,  0.03, 0.18, 0.04, 0.00,   0.00)$ \\
        \cline{3-3}
        & & CNN: $\hat{Y}_{U_{1.63}} = \{\mathrm{cat},\mathrm{bird}\}$\\
        && $\Prob: (0.00,   0.00,   0.22, 0.52, 0.13, 0.06, 0.01, 0.06, 0.00,   0.00, )$ \\
        \cline{2-3}
        & \multirow{4}{*}{Noisy $\mathbf{D}$} & 
        BNN: $\hat{Y}_{U_{1.63}} = \{\mathrm{bird}, \mathrm{cat}\}$\\
        && $\Prob: (0.00,   0.00,   0.42, 0.35, 0.05, 0.07, 0.01, 0.09, 0.00,   0.00)$\\
        \cline{3-3}
        & & CNN: $\hat{Y}_{U_{1.63}} = \{\mathrm{cat},\mathrm{bird}\} $ \\
        & & $\Prob: (0.00,   0.00,   0.25,  0.58, 0.00,   0.16, 0.00,   0.00,   0.01, 0.00 )$\\
        \hline 
        \multicolumn{3}{c}{Image 1: Noisy version}\\
        \hline
        \multirow{8}{*}{\includegraphics[width=0.22\textwidth]{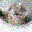}}  & \multirow{4}{*}{Clean $\mathbf{D}$} &  BNN: $\hat{Y}_{U_{1.63}} = \{\mathrm{frog}\}$\\
        && $\Prob: (0.00,   0.00,   0.08, 0.03, 0.03, 0.00,   0.86, 0.00,   0.00,   0.00)$\\
        \cline{3-3}
        & & CNN: $\hat{Y}_{U_{1.63}} = \{\mathrm{frog}\}$\\
        && $\Prob: (0.00,   0.00,   0.09, 0.11, 0.04, 0.00,   0.75, 0.00,   0.00,   0.01)$\\
        \cline{2-3}
        & \multirow{4}{*}{Noisy $\mathbf{D}$} & 
        BNN: $\hat{Y}_{U_{1.63}} = \{\mathrm{cat}, \mathrm{frog}\}$\\
        && $\Prob: (0.01, 0.00,   0.09, 0.55, 0.11, 0.03, 0.2,  0.01, 0.01, 0.00 )$\\
        \cline{3-3}
        & & CNN: $\hat{Y}_{U_{1.63}} = \{\mathrm{bird},\mathrm{cat}\} $ \\
        & & $\Prob: (0.00,   0.00,   0.61, 0.30,  0.03, 0.02, 0.00,   0.03, 0.01, 0.00 )$\\
        \hline 
        \multicolumn{3}{c}{Image 2: Clean version}\\
        \hline
        \multirow{8}{*}{\includegraphics[width=0.22\textwidth]{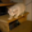}}  & \multirow{4}{*}{Clean $\mathbf{D}$} &  BNN: $\hat{Y}_{U_{1.63}} = \{\mathrm{dog}\}$\\
        && $\Prob: (0.00,   0.01, 0.00,   0.12, 0.00,   0.86, 0.00,   0.00,   0.02, 0.00)$\\
        \cline{3-3}
        & & CNN: $\hat{Y}_{U_{1.63}} = \{\mathrm{dog}\}$\\
        && $\Prob: (0.01, 0.00,   0.00,   0.00,   0.00,   0.97, 0.00,   0.01, 0.00,   0.01)$\\
        \cline{2-3}
        & \multirow{4}{*}{Noisy $\mathbf{D}$} & 
        BNN: $\hat{Y}_{U_{1.63}} = \{\mathrm{dog},\mathrm{cat}\}$\\
        && $\Prob: (0.00,   0.01, 0.00,   0.26, 0.00,   0.69, 0.00,   0.00,   0.00,   0.04)$\\
        \cline{3-3}
        & & CNN: $\hat{Y}_{U_{1.63}} = \{\mathrm{dog}\} $ \\
        & & $\Prob: (0.00,   0.00,   0.00,   0.00,   0.00,   1.00,   0.00,   0.00,   0.00,   0.00)$\\
        \hline 
        \multicolumn{3}{c}{Image 2: Noisy version}\\
        \hline
        \multirow{8}{*}{\includegraphics[width=0.22\textwidth]{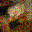}}  & \multirow{4}{*}{Clean $\mathbf{D}$} &  BNN: $\hat{Y}_{U_{1.63}} = \{\mathrm{frog}\}$\\
        && $\Prob: (0.00,   0.00,   0.10,  0.01, 0.13, 0.01, 0.73, 0.00,   0.02, 0.00 )$\\
        \cline{3-3}
        & & CNN: $\hat{Y}_{U_{1.63}} = \{\mathrm{frog}\}$\\
        && $\Prob: (0.00,   0.05, 0.02, 0.10,  0.07, 0.07, 0.66, 0.00,   0.01, 0.02)$\\
        \cline{2-3}
        & \multirow{4}{*}{Noisy $\mathbf{D}$} & 
        BNN: $\hat{Y}_{U_{1.63}} = \{\mathrm{cat},\mathrm{dog}\}$\\
        && $\Prob: (0.00,   0.02, 0.00,   0.34, 0.01, 0.31, 0.01, 0.01, 0.04, 0.26)$\\
        \cline{3-3}
        & & CNN: $\hat{Y}_{U_{1.63}} = \{\mathrm{dog}\} $ \\
        & & $\Prob: (0.00,   0.00,   0.00,   0.00,   0.00,   0.95, 0.00,   0.00,   0.00,   0.04)$\\
        \hline 
        \hline
    \end{tabular}
    \caption{The probabilistic prediction $\Prob(\cY \given \bx )$ and set-valued BOP $\hat{Y}_{U_{1.63}}$ \eqref{eq:set-valued BOP} of a few cat images from the CIFAR-10 data set. The chosen value of $\alpha$ is the median reward level.}
    \label{fig:set-valued BOPs of BNNs vs CNNs}
\end{figure}

Finally, the probabilistic predictions $\Prob(\cY \given \bx )$, and set-valued BOPs $\hat{Y}_{U^{0/1}_{2.05}}$ and $\hat{Y}_{U_{1.63}}$ \eqref{eq:set-valued BOP} for a cat image given in Fig.~\ref{fig:effect of reward-sensitivity on set-valued BOPs of BNNs vs CNNs} illustrate the way the reward-sensitivity can change the rank over the classes, which in turn, may lead to better predictive performance. On the probabilistic prediction $\Prob(\cY \given \bx)$ provided by the BNN, the class cat is ranked third. While set-valued BOPs $\hat{Y}_{U^{0/1}_{2.05}}$ do not include the class cat, the reward-sensitivity helps to change the rank over the classes to include the correct class cat into the set-valued BOP $\hat{Y}_{U_{1.63}}$. Indeed, for our chosen utility matrix, the score vector computed from Algorithm \ref{alg:set-valued BOP U_alpha revised} and the probability masses returned by the BNN are 
$$s=(0.13, 0.00,   0.24, 0.29, 0.18, 0.23, 0.18, 0.47, 0.16, 0.01).$$
One can readily see that the elements with the highest values are horse ($s=0.47$), cat ($s=0.29$), bird ($s=0.24$), etc., with ship arriving only at the $7$th position. On the probabilistic prediction $\Prob(\cY \given \bx)$ provided by the CNN, the class cat is ranked fifth by the score vector 
$$s=(0.48, 0.01, 0.21, 0.13, 0.09, 0.16, 0.08, 0.09, 0.17, 0.01),$$ and unfortunately, the set-valued BOPs cannot cover it, as probability is concentrated on non-animal elements.       

\begin{figure}[http!]
    \centering
    \begin{tabular}{c|l|l}
        \hline
        \multirow{6}{*}{\includegraphics[width=0.22\textwidth]{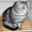}}  & \multirow{3}{*}{BNN} &  $\Prob: (0.13, 0.00,   0.08, 0.15, 0.00,   0.07, 0.00,   0.39, 0.16, 0.01)$\\
        \cline{3-3}
        && $\hat{Y}_{U^{0/1}_{2.05}} = \{\mathrm{horse},\mathrm{ship}\}$\\
        && $\hat{Y}_{U_{1.63}} = \{\mathrm{horse},\mathrm{cat}\}$\\
        \cline{2-3}
        & \multirow{3}{*}{CNN} & $\Prob: (0.48, 0.01, 0.17, 0.05, 0.00,   0.10,  0.00,   0.00,   0.17, 0.01)$\\
        \cline{3-3}
        & & $\hat{Y}_{U^{0/1}_{2.05}} = \{\mathrm{airplane},\mathrm{bird}\}$\\
        & & $\hat{Y}_{U_{1.63}} = \{\mathrm{airplane}\} $ \\
        \hline 
    \end{tabular}
    \caption{The probabilistic prediction $\Prob(\cY \given \bx )$ and set-valued BOPs $\hat{Y}_{U^{0/1}_{2.05}}$ and $\hat{Y}_{U_{1.63}}$ \eqref{eq:set-valued BOP} of a clean cat image from the CIFAR-10 data set. The ensembles are trained using clean training data $\mathbf{D}$. The chosen values of $\alpha$ are the median reward levels.}
    \label{fig:effect of reward-sensitivity on set-valued BOPs of BNNs vs CNNs}
\end{figure} 

\subsubsection{Sensitivity to noisy data}

To seek insights about this behavior of BNNs and CNNs with the dropout prediction option when encountering mismatches between training and test data types, we report calibration errors, which are computed on the representative probability distributions provided by the two types of compact ensembles. The results given in Appendix \ref{sec:Additional experimental results} suggest that BNNs tend to have more promising calibration errors. This, in turn, suggests that the representative probability distributions provided by BNNs may be more reliable and thus facilitate the performance of making both singleton and set-valued BOPs when encountering the mismatches between the training and test data types. 

\subsection{Practical recommendations}

We shall provide a few practical recommendations based on the results achieved so far.

The theoretical results on the desirable properties and computational efficiency, together with the consistent ability to gain the trade-off between recall and precision, might indeed motivate the use of our proposal for making set-valued BOPs wherever this trade-off is beneficial. Moreover, we encourage the end-user to participate in the design of the reward-sensitive utility-discounted accuracies \eqref{eq:Reward-sensitive-U-alpha} by estimating the base reward-sensitive accuracies and selecting the reward level $\alpha$. We further suggest that the end-user analyze the trade-off (e.g., on a validation set) and select the reward level $\alpha$ at their convenience.

If noisy test data can be encountered frequently enough, we suggest that the end-user train multiple (specialist) models to handle clean test data and noisy test data separately. Switching between these models might be monitored using additional information. For example, one might try to allow the UAV, which detects Plant Leaf Diseases in farms, to switch to (specialist) models to handle noisy test data when encountering severe weather. If training and monitoring a single predictive model is more desirable, the end-user might think of training the model on both clean and noisy data, which can be either collected or generated using data augmentation techniques. 

Finally, we suggest that the end-user can select a well-calibrated ensemble, which produces reliable representative distributions \eqref{eq:representative distribution}, when deploying our framework.   

\section{Conclusion}\label{sec:Conclusion}

We propose a compact cautious deep ensembling framework to tackle visible challenges in deep ensemble learning, where deep neural networks serve as ensemble members: training and storage burdens, and robustness of cautious (set-valued) predictions targeting multiple utilities, which may involve reward-sensitivity. We present both the generation and aggregation phases of the proposed framework, with a focus on the latter phase, where we propose a family of reward-sensitive utility-discounted accuracies with various desirable properties and efficient computations of set-valued BOPs. To our knowledge, we are the first who satisfactorily achieve both aspects in this setting. Along the way, we also provide a few new results for the simpler case of the conventional utility-discounted accuracies.   

We conduct various experiments on three image data sets to illustrate potential applications of the proposed framework. First, the empirical evidence suggests that our framework can provide a promising way to balance recall and prediction of set-valued predictions when reward-sensitivity is important. Moreover, it provides the end-user with an intuitive way to customize the reward-sensitive discounted accuracies at their convenience. Second, the empirical evidence suggests that different implementations of the proposed framework may be advantageous in different aspects. Notably, while defining the representative distribution using the squared Euclidean distance tends to be in favor of recall, doing so with the $L_1$ distance and the KL divergence tends to be in favor of precision. This also suggests that improving the proposed framework by exploring the potential (dis)advantages of other distances might be worthwhile.  

Although investigating the robustness of BNNs and CNNs with the Monte-Carlo dropout prediction under distribution shifts is not a primary goal of this paper, the empirical evidence suggests that BNNs tend to be (significantly) more robust. However, distribution shifts often significantly downgrade the predictive performance. We therefore suggest an extensive study on the use of multiple models, one per group of distribution shift scenarios, as another potential follow-up work. Such a work may greatly benefit high-stakes applications such as autonomous vehicles and UAVs.  

\backmatter

\bmhead{Acknowledgements}
This work was partially funded/supported by the Junior Professor Chair in Trustworthy AI (Ref.ANR-R311CHD). This work was also sponsored by the Office of Naval Research (ONR) and the Office of Naval Research Global (ONRG) under grant number N62909-23-1-2058.

\newpage

\section*{Declarations}

\subsection*{Author contributions statement}

K.D.T contributed to Methodology, Software, Validation, Visualization, and Writing – original draft. D.M.N contributed to Methodology, Software, Validation, Visualization, and Writing – original draft. V.L.N contributed to Conceptualization, Methodology, Software, Validation, Visualization, Supervision, Funding acquisition, and Writing – original draft. X.T.H contributed to Software, Validation, Visualization, and Writing – review \& editing. S.D. contributed to Conceptualization, Methodology, Supervision, and Writing – review \& editing. V.N.H contributed to Conceptualization, Methodology, Supervision, and Writing – review \& editing.

\subsection*{Data availability}

The 3 data sets used in the empirical study are all either publicly available at \href{https://data.mendeley.com/datasets/tywbtsjrjv/1}{https://data.mendeley.com/datasets/tywbtsjrjv/1} or can be downloaded by using \href{https://docs.pytorch.org/docs/2.13/data.html}{https://docs.pytorch.org/docs/2.13/data.html} . The source code has been made available at \href{https://github.com/kdimtensor/Compact-cautious-deep-ensembling-in-multi-class-classification}{https://github.com/kdimtensor/Compact-cautious-deep-ensembling-in-multi-class-classification }. 

\subsection*{Research funding}

Kim-Dung Tran received funding from Agence Nationale de la Recherche; Grant ID ANR-R311CHD. Dang-Man Nguyen received funding from Office of Naval Research Global; Grant ID N62909-23-1-2058. Vu-Linh Nguyen received funding from Agence Nationale de la Recherche; Grant ID ANR-R311CHD. Xuan-Truong Hoang received funding from Office of Naval Research Global; Grant ID N62909-23-1-2058. Van-Nam Huynh received funding from Office of Naval Research Global; Grant ID N62909-23-1-2058.

\subsection*{Conflict of interest}

We declare that we have no conflict of interest.

\subsection*{Ethics approval}

We declare that this research did not require Ethics approval.

\subsection*{Consent for publication}

All the authors of this manuscript consent to its publication.

\begin{appendices}

\section{Notation and acronyms}\label{app:Notation and acronyms}

\newcolumntype{R}{>{\raggedright\arraybackslash}p{0.65\textwidth}}

\begin{longtable}{l|Rp{\textwidth}}
\toprule
Symbol/Acronym & Meaning \\
\hline
$\cX$, $\bx$ & Instance space, instance \\
$\cY$, $Y$, $\hat{y}, \hat{y}_u$ & Output space, set prediction, BOP, BOP with utility $u$ \\
$\hat{Y}_{U^{0/1}_{\alpha}}$ & Set-valued BOP produced by utility-discounted accuracy function $U^{0/1}_{\alpha}$ \\
$X^p$, $Y^k$ & Feature, class variable \\
$\cY^k$ & Output space of $Y^k$ \\
$\cY^k_R \defi \cY^k \cup \{\cY^k\}$ & Set of predictions with rejection\\
$\cY^k_C \defi 2^{\cY^k} \setminus \{\emptyset\}$ & Set of predictions with partial cautiousness \\
$\cY_R \defi \cY^1_R \times \ldots \times \cY^K_R $ & Set of predictions with partial abstention \\
$\cY_C \defi \cY^1_C \times \ldots \times \cY^K_C $ & Set of predictions with partial cautiousness \\
$\cD = \{ (\bx_n,y_n)\}, \mathbf{T} = \{ (\bx_t,y_t)\}$  & Training (test) dataset of $n$ ($t$) elements \\
$\Delta^K, \Delta^{|\cY|}$ & Probability simplex of $K$ (cardinality of $\cY$) classes \\
$d, d_{\text{sE}}, d_{\text{L1}}, d_{\text{KL}}$ & Statistical distances using squared Euclidean, L1 norm, Kullback–Leibler divergence \\
\hline
$\cY_{set}, \cY_{per}, \cY_{sin}$ & Power set of $\cY$ excluding $\emptyset$, set of permissible predictions, set of singleton predictions \\
\hline
$u(\bar{y},y)$ & Utility function with predicted $\bar{y}$ and true $y$ \\
$u(y^i,y^j)$ & Utility function with predicted $y^i$ and true $y^j$ \\
$u^{0/1}(\cdot)$ & Utility function used for $0/1$ accuracy \\
$g(\cdot), g(|\bar{Y}|)$ & Reward function, reward function with predicted $\mathbf{Z}$ \\
$\alpha, \alpha^{\prime}, \underline{\alpha}, \overline{\alpha}$ & Hyperparameters controlling level of rewards \\
$U, U^{0/1}, U_{\alpha}, U_{\alpha}^{0/1}$ & Utility-discounted accuracy for set predictions \\
$U(\mathbf{Z}, \vec{y}),U^A(\mathbf{Z}, \vec{y})$ & Utility-discounted function with predicted $\mathbf{Z}$ and true $\vec{y}$, average utility-discounted function with predicted $\mathbf{Z}$ and true $\vec{y}$ \\
$U^{0/1}_{\alpha}$ & Utility-discounted function calculated with $\alpha$ for $0/1$ accuracy \\
$f_{U}^{k}, f_{U}^{K}, f_{U}^{k^*}$ & Local expected score with $U$ at $k$-th and $K$-th classes, local maximum expected score at $k^*$ class  \\
$\text{ECE}_{\text{classwise}}, \text{ECE}_{\text{confidence}}, \text{ECE}_{k}$ & Classwise and confidence estimated calibration errors, estimated calibration error of class $k$-th \\
$\text{MCE}_{\text{classwise}}, \text{MCE}_{\text{confidence}}$ & Classwise and confidence maximum calibrated errors \\
\hline
$K$, $P$ & Number of class variables, number of features\\
$\llbracket \cdot \rrbracket$, $[n], {K\choose k} $ & Indicator function, set $\{ 1,\ldots, n\}$ of natural numbers, $k$-combinations of $K$ elements \\
\hline
$\cY^{\cX}$ & Set of all possible hypothesis spaces \\
$\mathcal{H}, \mathbf{H}, \bh^j, \bh_d$ & Hypothesis space, ensemble of classifiers, classifier $j$-th in $\mathbf{H}$, representative classifier calculated by $d$ \\
clf\_sE, clf\_L1, clf\_KL & Representative classifiers obtained using $d_{\text{sE}}, d_{\text{L1}}, d_{\text{KL}}$ \\
$\mathbf{H}(\bx)$ & A set of outputs of the ensemble $\mathbf{H}$ \\
$p^j_k$ & Predicted probability for class $k$ produced by the $j$-th ensemble member \\
$\Prob(\cY \given \cX)$ & Joined probability distribution of output and input spaces \\
$\Prob(\mathbf{X}, Y)$ & Joined probability distribution of feature and class variables \\
$\Prob(\cY \given \bx)$ & Conditional probability distribution over $\cY$ given $\bx$ \\
$\Prob(Y \given \bx)$ & Probability of outcome $Y$ given $\bx$ \\
$\Prob(y^k \given \bx)$ & Marginal probability of relevance for outcome $Y^k = y^k $ given $\bx$\\ 
\hline
$\mathbb{R}_{+}, O(\cdot)$ & Positive real numbers, computational complexity \\
$\mathbf{R}_{i,j}$, $r_{i,j}$ & Reward-sensitivity matrix, reward element \\
$\mathbf{R}_{\bar{Y},j}$, $r_{\bar{Y},j}$ & Reward-sensitivity matrix, reward element in set-valued prediction setting \\
$\mathbf{E}_{\Prob(\cY \given \bx)}\big[u(\bar{y},y) \big]$ & Expected value of utility $u$ and distribution $\Prob$ \\
ual, rec, pro, siz, sin, set & Average (non-reward-)reward-sensitivity sensitivity utility-discounted accuracies, average recall, proportion of non-singleton predictions, average set size, proportion of correct singleton predictions, proportion of correct set-valued predictions \\
\hline
MCC & Multi-class classification \\
BOP & Bayes-optimal prediction \\
DNN & Deep neural network \\
CNN & Convolutional neural network \\
BNN & Bayesian neural network \\
UAV & Unmanned aerial vehicle \\
\hline
\caption{Notation and acronyms}\label{tab:Notation_and_acronyms}
\end{longtable} 

\section{Proofs for the formal results}

\subsection{Proof for Proposition \ref{pro:set-valued BOP U_alpha}}

By definition, finding a set-valued BOP $\hat{Y}_{U_\alpha}$ \eqref{eq:set-valued BOP} of $U_{\alpha}$ becomes 
\begin{align*}
\hat{Y}_{U_\alpha}  &\in  \operatorname*{argmax}_{\bar{Y} \in \cY_{\mathrm{set}}}\mathbf{E}_{\bp(\cY \given x)} \left[  U_{\alpha}(\bar{Y},y) \right]\\
&= \operatorname*{argmax}_{\bar{Y} \in \cY_{\mathrm{set}}}\sum_{y \in \cY} U_{\alpha}(\bar{Y},y) \, \Prob(y \given \bx) \\
&= \operatorname*{argmax}_{\bar{Y} \in \cY_{\mathrm{set}}}\sum_{y \in \cY} U_{\alpha}(\bar{Y},y)\, \Prob(y \given \bx)\\
& = \operatorname*{argmax}_{\bar{Y} \in \cY_{\mathrm{set}}}g_{\alpha}(|\bar{Y}|) \sum_{y \in \cY} \left[\sum_{\bar{y} \in \bar{Y}} u(\bar{y},y) \right]\, \Prob(y \given \bx) \\
& = \operatorname*{argmax}_{\bar{Y} \in \cY_{\mathrm{set}}}g(|\bar{Y}|) \sum_{y \in \cY} \left[\sum_{\bar{y} \in \bar{Y}} u(\bar{y},y) \, \Prob(y \given \bx) \right]\\
& = \operatorname*{argmax}_{\bar{Y} \in \cY_{\mathrm{set}}}g(|\bar{Y}|) \sum_{\bar{y} \in \bar{Y}} \left[\sum_{y \in \cY} u(\bar{y},y)  \, \Prob(y \given \bx) \right]\\
& = \operatorname*{argmax}_{\bar{Y} \in \cY_{\mathrm{set}}}g(|\bar{Y}|) \sum_{\bar{y} \in \bar{Y}} \mathbf{E}_{\bp(\cY \given x)}\left[u(\bar{y},y) \right]\,.
\end{align*}

Clearly, for each possible cardinality $k \in [K]$, the local optimal set-valued prediction of cardinality $k$ should contain the $k$ classes with the highest expected scores $\mathbf{E}_{\bp(\cY \given x)}\left[u(\bar{y},y) \right]$. Therefore, finding a set-valued BOP $\hat{Y}_{U_\alpha}$ \eqref{eq:set-valued BOP} of $U_{\alpha}$ can be done by computing $\mathbf{E}_{\bp(\cY \given x)}\left[ u(\bar{y},y)\right]$ for all the $\bar{y} \in \cY$ (in $O(K^2)$), sorting the classes in the decreasing order of the scores $\mathbf{E}_{\bp(\cY \given x)}\left[u(\bar{y},y)\right]$ (in $O(K \log(K))$), and incrementally including the top-ranked classes until the expected score $\mathbf{E}_{\bp(\cY \given x)} \left[ U(\bar{\by},y) \right]$ stops increasing (in $O(K)$). 

Altogether, finding a set-valued BOP $\hat{Y}_{U_\alpha}$ \eqref{eq:set-valued BOP} of $U_{\alpha}$ can be done in $O(K^2)$ using Algorithm \ref{alg:set-valued BOP U_alpha}.

\subsection{Proof for Proposition \ref{propos:desirable_properties_for_u_alpha}}

A proof for proposition \ref{propos:desirable_properties_for_u_alpha}  consists of the proofs for the following Lemmas.

\begin{lemma}[Normalization]
Let $u$ be any reward-sensitive accuracy $u$, where, for any true class $y \in \cY$, $1= u(y,y) > u(\bar{y},y) \geq 0$, $\bar{y} \neq y \in \cY$. Let $\alpha \in [-1/(|\cY|-1), \overline{\alpha}_u]$, where $\overline{\alpha}_u$ can be found in $O(K^3)$ using Algorithm \ref{alg:find max alpha}. The reward-sensitive utility-discounted accuracy $U_{\alpha}$ \eqref{eq:Reward-sensitive-U-alpha} satisfies the normalization property.
\end{lemma}
\begin{proof}
At first, we look at when $U_\alpha(\bar{Y},y) \geq 0$, i.e., $\forall (\bar{Y},y) \in \mathcal{Y}_{\mathrm{set}} \times \mathcal{Y}$,
\begin{equation}\label{eq:at_least_0}
    U_\alpha(\bar{Y},y) = g_\alpha(|\bar{Y}|) \sum_{\bar{y}\in \bar{Y}} u(\bar{y},y) = \left(\alpha + \frac{1 - \alpha}{|\bar{Y}|}\right) \frac{1}{|\bar{Y}|} \sum_{ \bar{y} \in \bar{Y}} u(\bar{y}, y) \geq 0 \,.
\end{equation}
Since for any possible cardinality $k \in [K]$, we have 
\begin{equation}
    \sum_{\bar{y}\in \bar{Y}} u(\bar{y},y) \geq 0\,,
\end{equation}
to ensure the relation \eqref{eq:at_least_0}, we should have, $\forall \bar{Y} \in \mathcal{Y}_{\mathrm{set}}$,
\begin{align*}
   \left(\alpha + \frac{1 - \alpha}{|\bar{Y}|}\right) \frac{1}{|\bar{Y}|} \geq 0  & \Leftrightarrow \frac{\alpha}{|\bar{Y}|} + \frac{1-\alpha}{|\bar{Y}|^2} \geq 0 \\
    & \Leftrightarrow \alpha |\bar{Y}| + 1 - \alpha \geq 0 \\
    & \Leftrightarrow \alpha (|\bar{Y}| - 1) +1 \geq 0 \,.
\end{align*}
Since for any singleton prediction $|\bar{Y}|$, we have $\alpha (|\bar{Y}| - 1) +1  = 1$ for any value of $\alpha$, we only need to ensure the relation $\alpha (|\bar{Y}| - 1) +1 \geq 0$ for the non-singleton predictions. For any non-singleton prediction $|\bar{Y}|$, we have 
\begin{equation}
    \alpha (|\bar{Y}| - 1) +1 \geq 0  \Leftrightarrow \alpha \geq \frac{-1}{|\bar{Y}| - 1} \,,
\end{equation}
which is an increasing function of $|\bar{Y}|$, the relation \eqref{eq:at_least_0} is ensured by any $\alpha \geq \nicefrac{-1}{(|\cY| -1)}$. 

We now look at when $U_{\alpha}(\bar{Y},y) \leq 1$, i.e., $\forall (\bar{Y},y) \in \mathcal{Y}_{\mathrm{set}} \times \mathcal{Y}$,
\begin{equation}\label{eq:at_most_1}
    U_\alpha(\bar{Y},y) =  \left(\alpha + \frac{1 - \alpha}{|\bar{Y}|}\right) \frac{1}{|\bar{Y}|} \sum_{ \bar{y} \in \bar{Y}} u(\bar{y}, y) \leq 1 \,.
\end{equation}

For any $\bar{Y}$ with $U_\alpha(\bar{Y},y) = 0$, it is clear that the relation \eqref{eq:at_most_1} is satisfied by any value of $\alpha$. Therefore, we only need to care about $\bar{Y}$ with positive reward $U_\alpha(\bar{Y},y) > 0$. For any $\bar{Y}$ with positive reward, we have 
\begin{align}
    U_\alpha(\bar{Y},y) := g_\alpha(|\bar{Y}|) \sum_{\bar{y}\in \bar{Y} } u(\bar{y},y) \leq 1 \Leftrightarrow  g_\alpha(|\bar{Y}|) \leq \frac{1}{\sum_{\bar{y}\in \bar{Y}} u(\bar{y},y)}\,.
\end{align}

Since for any possible cardinality $k$, $g_\alpha(|\bar{Y}|)$ is the same for any $\bar{Y}$ with $|\bar{Y}| = k$, we can search for the feasible range of $\alpha$, whose values satisfy the relation \eqref{eq:at_most_1}, by first determining the $K$ local feasible ranges, one per possible cardinality, and then returning their intersection as the final feasible range of $\alpha$. 

For each possible cardinality $k \in [K]$, since $\nicefrac{1}{\left(\sum_{\bar{y}\in \bar{Y}} u(\bar{y},y)\right)}$ is a decreasing function in $\sum_{\bar{y}\in \bar{Y}} u(\bar{Y},y)$, the relation \eqref{eq:at_most_1} is satisfied (locally) when 
\begin{align}
    g_\alpha(k) \leq & s(k) \defi  \min_{y \in \cY}\left(\min_{\substack{\bar{Y} \in \cY_{\mathrm{set}} \\ |\bar{Y}| = k}}\frac{1}{\sum_{\bar{y} \in \bar{Y}} u(\bar{y},y)}\right) = \min_{y \in \cY}\left(\frac{1}{\sum_{i = 1}^k u(\bar{y}^{(i),y},y)}\right) \,, 
\end{align}
where, for any $y\in \cY$, $\bar{y}^{(i),y}$ the $i$-th class after sorting the classes $\bar{y} \in \cY$ in the decreasing order of $u(\bar{y},y)$. This is equivalent to 
\begin{align}
    g_\alpha(k) \leq s(k) & \Leftrightarrow \frac{\alpha}{k} + \frac{1-\alpha}{k^2} \leq s(k) \\
    & \Leftrightarrow \alpha k + 1-\alpha \leq s(k) k^2 \\
    & \Leftrightarrow \alpha (k-1) + 1 \leq s(k) k^2 \\
    & \Leftrightarrow \alpha (k-1)  \leq s(k) k^2 - 1 \\
    & \Leftrightarrow \alpha  \leq \frac{s(k) k^2 - 1}{k-1} := \alpha_k\,.
\end{align}

Once the values $\alpha_k$, $k\in [K]$, are found, the range $[-1/(|\cY|-1), \overline{\alpha}_u]$, where for any $\alpha \in [-1/(|\cY|-1), \overline{\alpha}_u]$, $U_\alpha$ satisfies the normalization property, can be found by setting $\overline{\alpha}_u$ to be the smallest $\alpha_k$, $k\in [K]$.

Altogether, the range $[1, \overline{\alpha}_u]$, can be found in $O(K^3)$ using Algorithm \ref{alg:find max alpha}.

Finally, we show that the feasible range $[-1/(|\cY|-1), \overline{\alpha}_u]$ is always non-empty. For any $k \in [K]$, by definition $s(k) > \nicefrac{1}{k}$. Therefore, we have 
\begin{align}
    \frac{s(k) k^2 - 1}{k-1} > \frac{\frac{1}{k} k^2 - 1}{k-1} \geq \frac{k - 1}{k-1}  \geq 1 > \frac{-1}{|\cY|-1} \,.
\end{align}
This ensures that $\overline{\alpha}_u > 1$.
\end{proof}

\begin{lemma}[Universal feasibility]
 Let $u$ be any reward-sensitive accuracy $u$, where, for any true class $y \in \cY$, $1= u(y,y) > u(\bar{y},y) \geq 0$, $\bar{y} \neq y \in \cY$. Let $\alpha \in (1, \overline{\alpha}_u)$. The reward-sensitive utility-discounted accuracy $U_{\alpha}$ \eqref{eq:Reward-sensitive-U-alpha} satisfies the universal feasibility property.   
\end{lemma}

\begin{proof}
As elaborated in \citep{yang2016costs} Property \ref{proper:Universal feasibility} (universal feasibility), a reward-sensitivity utility-discounted accuracy $U_\alpha(\bar{Y},y)$ allows indeterminate predictions if, $\forall (\bar{Y},y) \in \mathcal{Y}_{\mathrm{set}} \times \mathcal{Y}$ with $|\bar{Y}| \geq 2$ and $\sum_{\bar{y}\in \bar{Y}} u(\bar{y},y) > 0$, we have 
\begin{equation}
    U_\alpha(\bar{Y},y) := g(|\bar{Y}|) \sum_{\bar{y} \in \bar{Y}}u(\bar{y},y) > U^{\mathrm{ave}}_\alpha(\bar{Y},y) := \frac{1}{|\bar{Y}|} \sum_{\bar{y} \in \bar{Y}}u(\bar{y},y) \,,
\end{equation}
where $U^{\mathrm{ave}}_\alpha(\bar{Y},y)$ is the average of the reward-sensitive accuracy $u$. Since for possible cardinality $k \geq 2$, there is at least one $\bar{Y}$ with $|\bar{Y}| = k$ and $\sum_{\bar{y}\in \bar{Y}}u(\bar{y},y) > 0$, therefore,
\begin{align*}
U_\alpha(\bar{Y},y) > U^{\mathrm{ave}}_\alpha(\bar{Y},y) &\Leftrightarrow g(|\bar{Y}|)\sum_{\bar{y} \in \bar{Y}}u(\bar{y},y) > \frac{1}{|\bar{Y}|} \sum_{\bar{y} \in \bar{Y}}u(\bar{y},y) \\
& \Leftrightarrow \frac{\alpha}{|\bar{Y}|} + \frac{1-\alpha}{|\bar{Y}|^2} > \frac{1}{|\bar{Y}|} \\
& \Leftrightarrow \alpha |\bar{Y}| + 1 - \alpha > |\bar{Y}| \\
& \Leftrightarrow \alpha (|\bar{Y}|-1) > |\bar{Y}| - 1 \\
& \Leftrightarrow \alpha > 1 \,.
\end{align*}

We now show that for any chosen $\alpha \in (1, \overline{\alpha}_u)$, for any possible prediction $\bar{Y} \in \cY_{\mathrm{set}}$, there is at least one distribution $\boldsymbol{p}(\mathcal{Y}\mid \bx) \in \Delta^{|\mathcal{Y}|}$ such that 
\begin{equation}
\hat{Y}_{U_{\alpha}} \in \operatorname*{argmax}_{\bar{Y}\in \mathcal{Y}_{\mathrm{set}}} \mathbf{E}_{\Prob(\cY \given \bx)}\left[U_\alpha(\bar{Y},y) \right] = \operatorname*{argmax}_{\bar{Y}\in \mathcal{Y}_{\mathrm{set}}} g_\alpha(|\bar{Y}|) \sum_{y \in \cY} \left[ \sum_{\bar{y} \in \bar{Y}} u(\bar{y},y) \right] \Prob(y \given \bx) \,.   
\end{equation}
For any singleton prediction $\bar{Y} = \{y\}$, we have $g_\alpha(|\bar{Y}|) = 1$, Therefore, we can simply pick up the degenerate distribution $\boldsymbol{p}(\mathcal{Y}\mid \bx)$ with $\boldsymbol{p}(y\mid \bx) = 1$ to have $\mathbf{E}_{\boldsymbol{p}(y\mid \bx)}\left[U_\alpha(\{y\},y) \right] = 1$, which in turn ensures that $\{y\}$ is a set-valued BOP on the distribution $\boldsymbol{p}(\mathcal{Y}\mid \bx)$. 

We now show that there is at least one distribution $\boldsymbol{p}(\mathcal{Y}\mid \bx) \in \Delta^K$, where a BOP of $U_{\alpha}$ should be non-singleton. Without the loss of generality, let us assume that $\bar{Y} = \{y^1, y^2\}$ is the pair of classes with the highest reward for making a wrong prediction, i.e., $u(y^1,y^2) \geq  u(y^i,y^j)$, for any $(1,2) \neq (i,j)$ such that $i\neq j$. Let $\boldsymbol{p}(\mathcal{Y}\mid \bx)$ be some distribution, where $\boldsymbol{p}(y^1\mid \bx) = p$ and $\boldsymbol{p}(y^2\mid \bx) = 1-p$. We have 
\begin{align*}
    \mathbf{E}_{\Prob(\cY \given \bx)}\left[U_\alpha(\{y^1\},y) \right] &= p \cdot u(y^1,y^1) + (1-p)\cdot u(y^1,y^2) = p + (1-p) \cdot u(y^1,y^2) \\
    \mathbf{E}_{\Prob(\cY \given \bx)}\left[U_\alpha(\{y^2\},y) \right] &= p \cdot u(y^2,y^1) + (1-p) \cdot u(y^2,y^2) = p \cdot u(y^2,y^1) + (1-p) \\
    \mathbf{E}_{\Prob(\cY \given \bx)}\left[U_\alpha(\bar{Y},y) \right]  &= g_\alpha(2)\cdot \left( \mathbf{E}_{\Prob(\cY \given \bx)}\left[U_\alpha(\{y^1\},y) \right] + \mathbf{E}_{\Prob(\cY \given \bx)}\left[U_\alpha(\{y^2\},y) \right]\right)\\
    & =g_\alpha(2) \cdot \left(p \cdot \left(1 + u(y^2,y^1)\right) + (1-p) \cdot \left(u(y^1,y^2) + 1\right) \right) 
\end{align*}
For any $y^k \not\in \bar{Y}$, we have 
\begin{equation*}
    \mathbf{E}_{\Prob(\cY \given \bx)}\left[U_\alpha(\{y^k\},y) \right] = p \cdot u(y^k,y^1) + (1-p)\cdot u(y^k,y^2) \,.
\end{equation*}
It is easy to see that any $y^k \not\in \bar{Y}$ can not be a BOP because the relation
\begin{equation*}
    \mathbf{E}_{\Prob(\cY \given \bx)}\left[U_\alpha(\{\bar{y^1}\},y) \right] \geq \mathbf{E}_{\Prob(\cY \given \bx)}\left[U_\alpha(\{\bar{y^k}\},y) \right] \,,
\end{equation*}
is ensured by 
\begin{align*}
    & \mathbf{E}_{\Prob(\cY \given \bx)}\left[U_\alpha(\{\bar{y^1}\},y) \right] - \mathbf{E}_{\Prob(\cY \given \bx)}\left[U_\alpha(\{\bar{y^k}\},y) \right] \\
    =& p\cdot(\underbrace{1 - u(y^k,y^1)}_{>0}) + (1-p)\cdot (\underbrace{u(y^1,y^2)- u(y^k,y^2)}_{\geq 0}) \\
    >& 0 \,.
\end{align*}
We now consider the case 
\begin{equation*}
    p= \frac{1 - u(y^1,y^2)}{2 - u(y^1,y^2) - u(y^2,y^1)} \,.
\end{equation*}
Since $g_\alpha(2) = \nicefrac{(\alpha +1)}{4} > \nicefrac{1}{2}$, $\forall \alpha \in (1, \overline{\alpha}_u)$, and  
\begin{align*}
    \mathbf{E}_{\Prob(\cY \given \bx)}\left[U_\alpha(\{y^1\},y) \right] &= \mathbf{E}_{\Prob(\cY \given \bx)}\left[U_\alpha(\{y^2\},y) \right] \\
    &= \frac{1-u(y^1,y^2) + \left( 1- u(y^2,y^1)\right)\cdot u(y^1,y^2)}{2 - u(y^1,y^2) - u(y^2,y^1)}\,,
\end{align*}
we have 
\begin{equation*}
    \mathbf{E}_{\Prob(\cY \given \bx)}\left[U_\alpha(\{\bar{Y}\},y) \right] > \mathbf{E}_{\Prob(\cY \given \bx)}\left[U_\alpha(\{\bar{y}\},y) \right] \,, \bar{y} \in \bar{Y} = \{y^1, y^2\}\,.
\end{equation*}
In other words, both $y^1$ and $y^2$ can not be a BOP. Since Proposition \ref{pro:set-valued BOP U_alpha} ensures that a BOP should contain the classes with the highest expected scores $\mathbf{E}_{\Prob(\cY \given \bx)}\left[u(y^k,y) \right]$, we should have a BOP $Y \supset \bar{Y} = \{1, 2\}$.

We end this proof with an investigation on the specific case where, for any true class $y \in \cY$, $1= u(y,y) > 0.5 \geq u(\bar{y},y) \geq 0$, $\bar{y} \neq y \in \cY$. We show that for any non-singleton $\bar{Y}$ with $|\bar{Y}| = 2$, there is a non-singleton prediction $Y \supset \bar{Y}$, which is a set-valued BOP on at least one distribution $\boldsymbol{p}(\mathcal{Y}\mid \bx) \in \Delta^K$, where $\boldsymbol{p}(y\mid \bx) = 0$ if $y \not\in \bar{Y}$. Without the loss of generality, let us assume that $\bar{Y} = \{y^1, y^2\}$, and $\boldsymbol{p}(y^1\mid \bx) = p$ and $\boldsymbol{p}(y^2\mid \bx) = 1-p$. We again consider the case 
\begin{equation*}
    p= \frac{1 - u(y^1,y^2)}{2 - u(y^1,y^2) - u(y^2,y^1)} \,.
\end{equation*}
Since $g_\alpha(2) = \nicefrac{(\alpha +1)}{4} > \nicefrac{1}{2}$, $\forall \alpha \in (1, \overline{\alpha}_u)$, and  
\begin{align*}
    \mathbf{E}_{\Prob(\cY \given \bx)}\left[U_\alpha(\{y^1\},y) \right] &= \mathbf{E}_{\Prob(\cY \given \bx)}\left[U_\alpha(\{y^2\},y) \right] \\
    &= \frac{1-u(y^1,y^2) + \left( 1- u(y^2,y^1)\right)\cdot u(y^1,y^2)}{2 - u(y^1,y^2) - u(y^2,y^1)}\,,
\end{align*}
we have 
\begin{equation*}
    \mathbf{E}_{\Prob(\cY \given \bx)}\left[U_\alpha(\{\bar{Y}\},y) \right] > \mathbf{E}_{\Prob(\cY \given \bx)}\left[U_\alpha(\{\bar{y}\},y) \right] \,, \bar{y} \in \bar{Y} = \{y^1, y^2\}\,.
\end{equation*}
In other words, both $y^1$ and $y^2$ can not be a BOP. We now show that any $y^k \not\in \bar{Y}$, with 
\begin{equation*}
    \mathbf{E}_{\Prob(\cY \given \bx)}\left[U_\alpha(\{y^k\},y) \right] = \frac{\left(1-u(y^1,y^2)\right)\cdot u(y^k,y^1) + \left( 1- u(y^2,y^1)\right)\cdot u(y^k,y^2)}{2 - u(y^1,y^2) - u(y^2,y^1)}\,. 
\end{equation*}
can not be a BOP. This will be done by showing that we have the relation
\begin{equation*}
    \mathbf{E}_{\Prob(\cY \given \bx)}\left[U_\alpha(\{y^1\},y) \right] \geq     \mathbf{E}_{\Prob(\cY \given \bx)}\left[U_\alpha(\{y^k\},y) \right]
\end{equation*} 
is satisfied in both cases where $u(y^1,y^2) - u(y^k,y^2) \geq 0$ and $u(y^1,y^2) - u(y^k,y^2) < 0$. When $u(y^1,y^2) - u(y^k,y^2) \geq 0$, we have 
\begin{align*}
  &\left(2 - u(y^1,y^2) - u(y^2,y^1)\right) \left(\mathbf{E}_{\Prob(\cY \given \bx)}\left[U_\alpha(\{y^1\},y) \right] -     \mathbf{E}_{\Prob(\cY \given \bx)}\left[U_\alpha(\{y^k\},y) \right] \right)\\
 =& \left(1-u(y^1,y^2)\right)\cdot \left(1-u(y^k,y^1)\right) + \left( 1- u(y^2,y^1)\right)\cdot \left(\underbrace{u(y^1,y^2) - u(y^k,y^2) }_{\geq 0}\right) \\
\geq& 0 \,.
\end{align*}
When $u(y^1,y^2) - u(y^k,y^2) < 0$, we have 
\begin{align*}
  &\left(2 - u(y^1,y^2) - u(y^2,y^1)\right) \left(\mathbf{E}_{\Prob(\cY \given \bx)}\left[U_\alpha(\{y^1\},y) \right] -     \mathbf{E}_{\Prob(\cY \given \bx)}\left[U_\alpha(\{y^k\},y) \right] \right)\\
 =& \left(1-u(y^1,y^2)\right)\cdot \left(1-u(y^k,y^1)\right) + \left( \underbrace{1- u(y^2,y^1)}_{\leq 1}\right)\cdot \left(\underbrace{u(y^1,y^2) - u(y^k,y^2) }_{< 0} \right) \\
\geq& \left(1-u(y^1,y^2)\right)\cdot \left(1-u(y^k,y^1)\right) +  \left(u(y^1,y^2) - u(y^k,y^2) \right) \\
=&1-u(y^1,y^2) - \left(1-u(y^1,y^2)\right)\cdot u(y^k,y^1)  + u(y^1,y^2) - u(y^k,y^2) \\
=&1 - \left(1-u(y^1,y^2)\right)\cdot u(y^k,y^1) -  u(y^k,y^2) \\
=&1 - \left(\underbrace{1-u(y^1,y^2)}_{\leq 1}\right)\cdot u(y^k,y^1) - u(y^k,y^2)  \\
\geq&1 - \underbrace{u(y^k,y^1)}_{\leq 0.5} - \underbrace{u(y^k,y^2)}_{\leq 0.5} \\
\geq&1 - 0.5 -  0.5 \\
\geq& 0 \,.
\end{align*}
Since Proposition \ref{pro:set-valued BOP U_alpha} ensures that a BOP should contain the classes with the highest expected scores $\mathbf{E}_{\Prob(\cY \given \bx)}\left[u(y^k,y) \right]$, we should have a BOP $Y \supset \bar{Y} = \{1, 2\}$.
\end{proof}

\begin{lemma}[Nested optimality]
Let $u$ be any reward-sensitive accuracy $u$, where, for any true class $y \in \cY$, $1= u(y,y) > u(\bar{y},y) \geq 0$, $\bar{y} \neq y \in \cY$. Let $\alpha \in (1, \overline{\alpha}_u)$. The reward-sensitive utility-discounted accuracy $U_{\alpha}$ \eqref{eq:Reward-sensitive-U-alpha} satisfies the nested optimality property.   
\end{lemma}
\begin{proof}
Proposition \ref{pro:set-valued BOP U_alpha} ensures that, for any possible value $k = |\bar{Y}|$, a local BOP 
\begin{equation}
   \hat{Y}^{k}_{U_\alpha} \in \underset{\substack{\bar{Y} \in \cY_{\mathrm{set}} \\ |\bar{Y}| = k}}{\arg\max}\, \mathbf{E}_{\Prob(\cY \given \bx)}(U_\alpha(\bar{Y},y))
\end{equation}
can be found by including exactly the $k$ top-ranked classes after sorting the classes $\bar{y}$ in the decreasing order of $\mathbf{E}_{\Prob(\cY \given \bx)}\left[u(\bar{y},y) \right]$. Therefore, $\bar{Y}_{k}^{U_\alpha} \subset \bar{Y}_{k'}^{U_\alpha}, \forall k \leq k'$.
\end{proof}

\begin{lemma}[Unimodality]\label{lem:Unimodality}
 Let $u$ be any reward-sensitive accuracy $u$, where, for any true class $y \in \cY$, $1= u(y,y) > u(\bar{y},y) \geq 0$, $\bar{y} \neq y \in \cY$. Let $\alpha \in (1, \overline{\alpha}_u)$. The reward-sensitive utility-discounted accuracy $U_{\alpha}$ \eqref{eq:Reward-sensitive-U-alpha} satisfies the unimodality property.   
\end{lemma}
\begin{proof}
For any \( k \in [K] \), let 
\begin{equation}
    \hat{Y}^{k}_{U_\alpha} \in \underset{\substack{\bar{Y} \in \cY_{\mathrm{set}} \\ |\bar{Y}| = k}}{\arg\max}\, \mathbf{E}_{\Prob(\cY \given \bx)}\left[U_\alpha(\bar{Y},y)\right] \,.
\end{equation}
Proposition \ref{pro:set-valued BOP U_alpha} ensures that 
\begin{equation}
        f_{U_\alpha}^k = \mathbf{E}_{\Prob(\cY \given \bx)}\left[U_\alpha(\bar{Y}^{k}_{U_\alpha},y)\right] = g_{\alpha}(k) \sum_{k' = 1}^k \mathbf{E}\left[u(\bar{y}^{(k'),y},y) \right]\,,
\end{equation}
where $\bar{y}^{(k')}$ is the $k^{\prime}$-th class after sorting the classes $\bar{y} \in \cY$ in the decreasing order of $\mathbf{E}_{\Prob(\cY \given \bx)}\left[u(\bar{y}^{(k'),y},y) \right]$. It is required the existence of a $k^*$ such that $ f_{U_\alpha}^1 \leq \dots \leq f_{U_\alpha}^{k^*} \land f_{U_\alpha}^{k^*} \geq \dots \geq f_{U_\alpha}^{K}$. This is equivalent to ensuring the following implications: \[
\begin{cases}
    f_{U_\alpha}^{k} \geq f_{U_\alpha}^{k-1} \Rightarrow f_{U_\alpha}^{k-1} \geq f_{U_\alpha}^{k-2} \,, \\
    f_{U_\alpha}^{k} \geq f_{U_\alpha}^{k+1} \Rightarrow f_{U_\alpha}^{k+1} \geq f_{U_\alpha}^{k+2}\,.
\end{cases}
\] 
In the remainder of this proof, we use the shorthand notation $\mathbf{E}$ instead of $\mathbf{E}_{\Prob(\cY \given \bx)}$.

We first prove that $f_{U_\alpha}^{k} \geq f_{U_\alpha}^{k-1} \Rightarrow f_{U_\alpha}^{k-1} \geq f_{U_\alpha}^{k-2}$. 
Since $g_{\alpha}(k)$ is $(1/x)$-convex, where $\nicefrac{1}{g_{\alpha}(k)}$ has the positive second derivative  
\begin{equation*}
    \frac{2\alpha^2}{\left(\alpha(k-1) +1\right)^3}\,,
\end{equation*}
we have
\begin{align*}
&\frac{1}{2}\left( \frac{1}{g_{\alpha}(k)} + \frac{1}{g_{\alpha}(k-2)} \right) \geq \frac{1}{g_{\alpha}(k-1)} \Leftrightarrow \frac{g_{\alpha}(k-1)}{g_{\alpha}(k)} + \frac{g_{\alpha}(k-1)}{g_{\alpha}(k-2)} \geq 2 \\
\Leftrightarrow& g_{\alpha}(k-1)g_{\alpha}(k-2) + g_{\alpha}(k)g_{\alpha}(k-1) \geq 2g_{\alpha}(k)g_{\alpha}(k-2) \\
\Leftrightarrow& g_{\alpha}(k-1)g_{\alpha}(k-2) - g_{\alpha}(k)g_{\alpha}(k-2) \geq g_{\alpha}(k)g_{\alpha}(k-2) - g_{\alpha}(k)g_{\alpha}(k-1) \\
\Leftrightarrow& g_{\alpha}(k-2)[ g_{\alpha}(k-1) - g_{\alpha}(k)] \geq g_{\alpha}(k)[g_{\alpha}(k-2) - g_{\alpha}(k-1)] \\
\Leftrightarrow& \frac{g_{\alpha}(k-2)}{g_{\alpha}(k-2) - g_{\alpha}(k-1)} \geq \frac{g_{\alpha}(k)}{g_{\alpha}(k-1) - g_{\alpha}(k)} \\
\Leftrightarrow& \frac{g_{\alpha}(k-1)}{g_{\alpha}(k-2) - g_{\alpha}(k-1)} + 1 \geq \frac{g_{\alpha}(k)}{g_{\alpha}(k-1) - g_{\alpha}(k)} \\
\Leftrightarrow& \frac{g_{\alpha}(k-1)}{g_{\alpha}(k-2) - g_{\alpha}(k-1)} \geq \frac{g_{\alpha}(k)}{g_{\alpha}(k-1) - g_{\alpha}(k)} - 1.
\end{align*}
This relation ensures the following relation
\begin{align*}
&f_{U_\alpha}^{k} \geq f_{U_\alpha}^{k-1} \\
\Leftrightarrow &g_\alpha(k) \sum_{k' = 1}^{k} \mathbf{E}\left[u(\bar{y}^{(k'),y},y)\right] \geq g_\alpha(k-1) \sum_{k' = 1}^{k-1} \mathbf{E}\left[u(\bar{y}^{(k'),y},y)\right] \\
\Leftrightarrow &g_\alpha(k) \mathbf{E}\left[u(\bar{y}^{(k),y},y)\right] \geq [g_\alpha(k-1) -g_\alpha(k)]\sum_{k' = 1}^{k-1} \mathbf{E}\left[u(\bar{y}^{(k'),y},y)\right] \\
\Leftrightarrow &\frac{g_{\alpha}(k)}{g_{\alpha}(k-1) -g_{\alpha}(k)} \mathbf{E}\left[u(\bar{y}^{(k),y},y)\right] \geq \sum_{k' = 1}^{k-1} \mathbf{E}\left[u(\bar{y}^{(k'),y},y)\right] \\
\Leftrightarrow &\frac{g_{\alpha}(k)}{g_{\alpha}(k-1) -g_{\alpha}(k)} \mathbf{E}\left[u(\bar{y}^{(k),y},y)\right] - \mathbf{E}\left[u(\bar{y}^{(k-1),y},y)\right] \geq \sum_{k' = 1}^{k-2} \mathbf{E}\left[u(\bar{y}^{(k'),y},y)\right] \\
\Rightarrow &\frac{g_{\alpha}(k)}{g_{\alpha}(k-1) -g_{\alpha}(k)} \mathbf{E}\left[u(\bar{y}^{(k-1),y},y)\right] - \mathbf{E}\left[u(\bar{y}^{(k-1),y},y)\right] \geq \sum_{k' = 1}^{k-2} \mathbf{E}\left[u(\bar{y}^{(k'),y},y)\right] \\
\Leftrightarrow &\left[ \frac{g_{\alpha}(k)}{g_{\alpha}(k-1) -g_{\alpha}(k)} - 1 \right] \mathbf{E}\left[u(\bar{y}^{(k-1),y},y)\right] \geq \sum_{k' = 1}^{k-2} \mathbf{E}\left[u(\bar{y}^{(k'),y},y)\right] \\
\Rightarrow &\left[ \frac{g_{\alpha}(k-1)}{g_{\alpha}(k-2) -g_{\alpha}(k-1)} \right] \mathbf{E}\left[u(\bar{y}^{(k-1),y},y)\right] \geq \sum_{k' = 1}^{k-2} \mathbf{E}\left[u(\bar{y}^{(k'),y},y)\right] \\
\Leftrightarrow &g_{\alpha}(k-1) \mathbf{E}\left[u(\bar{y}^{(k-1),y},y)\right] \geq [g_{\alpha}(k-2) - g_{\alpha}(k-1)] \sum_{k' = 1}^{k-2} \mathbf{E}\left[u(\bar{y}^{(k'),y},y)\right] \\
\Leftrightarrow &g_{\alpha}(k-1) \mathbf{E}\left[u(\bar{y}^{(k-1),y},y)\right] + g_{\alpha}(k-1) \sum_{k' = 1}^{k-2} \mathbf{E}\left[u(\bar{y}^{(k'),y},y)\right] \\
&\geq g_{\alpha}(k-2) \sum_{k' = 1}^{k-2} \mathbf{E}\left[u(\bar{y}^{(k'),y},y)\right] \\
\Leftrightarrow &g_{\alpha}(k-1) \sum_{k' = 1}^{k-1} \mathbf{E}\left[u(\bar{y}^{(k'),y},y)\right] \geq g_{\alpha}(k-2) \sum_{k' = 1}^{k-2} \mathbf{E}\left[u(\bar{y}^{(k'),y},y)\right] \\
\Leftrightarrow &f_{U_\alpha}^{k-1} \geq f_{U_\alpha}^{k-2}.
\end{align*}
   
We now prove that $f_{U_\alpha}^{k} \geq f_{U_\alpha}^{k+1} \Rightarrow f_{U_\alpha}^{k+1} \geq f_{U_\alpha}^{k+2}$. Since $g_{\alpha}(k)$ is $(1/x)$-convex, we also have:
\begin{align*}
&\frac{1}{2}\left( \frac{1}{g_{\alpha}(k)} + \frac{1}{g_{\alpha}(k+2)} \right) \geq \frac{1}{g_{\alpha}(k+1)} \Leftrightarrow \frac{g_{\alpha}(k+1)}{g_{\alpha}(k)} + \frac{g_{\alpha}(k+1)}{g_{\alpha}(k+2)} \geq 2 \\
\Leftrightarrow& g_{\alpha}(k+1)g_{\alpha}(k+2) + g_{\alpha}(k)g_{\alpha}(k+1) \geq 2g_{\alpha}(k)g_{\alpha}(k+2) \\
\Leftrightarrow& g_{\alpha}(k)g_{\alpha}(k+1) - g_{\alpha}(k)g_{\alpha}(k+2) \geq g_{\alpha}(k)g_{\alpha}(k+2) - g_{\alpha}(k+1)g_{\alpha}(k+2) \\
\Leftrightarrow& g_{\alpha}(k)[ g_{\alpha}(k+1) - g_{\alpha}(k+2)] \geq g_{\alpha}(k+2)[g_{\alpha}(k) - g_{\alpha}(k+1)] \\
\Leftrightarrow& \frac{g_{\alpha}(k)}{g_{\alpha}(k) - g_{\alpha}(k+1)} \geq \frac{g_{\alpha}(k+2)}{g_{\alpha}(k+1) - g_{\alpha}(k+2)}.
\end{align*}
This relation ensures the following relation:
\begin{align*}
&f_{U_\alpha}^{k} \geq f_{U_\alpha}^{k+1} \\
\Leftrightarrow &g_\alpha(k) \sum_{k' = 1}^{k} \mathbf{E}\left[u(\bar{y}^{(k'),y},y)\right] \geq g_\alpha(k+1) \sum_{k' = 1}^{k+1} \mathbf{E}\left[u(\bar{y}^{(k'),y},y)\right] \\
\Leftrightarrow &g_\alpha(k) \left[ \sum_{k' = 1}^{k+1} \mathbf{E}\left[u(\bar{y}^{(k'),y},y)\right] - \mathbf{E}\left[u(\bar{y}^{(k+1),y},y)\right] \right] \geq g_\alpha(k+1) \sum_{k' = 1}^{k+1} \mathbf{E}\left[u(\bar{y}^{(k'),y},y)\right] \\
\Leftrightarrow &[g_\alpha(k) -g_\alpha(k+1)]\sum_{k' = 1}^{k+1} \mathbf{E}\left[u(\bar{y}^{(k'),y},y)\right] \geq g_\alpha(k) \mathbf{E}\left[u(\bar{y}^{(k+1),y},y)\right] \\
\Leftrightarrow &\sum_{k' = 1}^{k+1} \mathbf{E}\left[u(\bar{y}^{(k'),y},y)\right] \geq \frac{g_{\alpha}(k)}{g_{\alpha}(k) -g_{\alpha}(k+1)} \mathbf{E}\left[u(\bar{y}^{(k+1),y},y)\right] \\
\Leftrightarrow &\sum_{k' = 1}^{k+1} \mathbf{E}\left[u(\bar{y}^{(k'),y},y)\right] \geq \frac{g_{\alpha}(k)}{g_{\alpha}(k) -g_{\alpha}(k+1)} \mathbf{E}\left[u(\bar{y}^{(k+2),y},y)\right] \\
\Leftrightarrow &\sum_{k' = 1}^{k+1} \mathbf{E}\left[u(\bar{y}^{(k'),y},y)\right] \geq \frac{g_{\alpha}(k+2)}{g_{\alpha}(k+1) -g_{\alpha}(k+2)} \mathbf{E}\left[u(\bar{y}^{(k+2),y},y)\right] \\
\Leftrightarrow &[g_{\alpha}(k+1) - g_{\alpha}(k+2)] \sum_{k' = 1}^{k+1} \mathbf{E}\left[u(\bar{y}^{(k'),y},y)\right] \geq g_{\alpha}(k+2) \mathbf{E}\left[u(\bar{y}^{(k+2),y},y)\right] \\
\Leftrightarrow &g_{\alpha}(k+1) \sum_{k' = 1}^{k+1} \mathbf{E}\left[u(\bar{y}^{(k'),y},y)\right] \geq g_{\alpha}(k+2) \left [ \mathbf{E}\left[u(\bar{y}^{(k+2),y},y)\right] + \sum_{k' = 1}^{k+1} \mathbf{E}\left[u(\bar{y}^{(k'),y},y)\right] \right] \\
\Leftrightarrow &g_{\alpha}(k+1) \sum_{k' = 1}^{k+1} \mathbf{E}\left[u(\bar{y}^{(k'),y},y)\right] \geq g_{\alpha}(k+2) \sum_{k' = 1}^{k+2} \mathbf{E}\left[u(\bar{y}^{(k'),y},y)\right] \\
\Leftrightarrow &f_{U_\alpha}^{k+1} \geq f_{U_\alpha}^{k+2}\,.
\end{align*}
\end{proof}

\subsection{Proof for Proposition \ref{propos:Nested coarsening}}

\begin{proof}
In this proof, we use the shorthand notation $\mathbf{E}$ instead of $\mathbf{E}_{\Prob(\cY \given \bx)}$.

We will first show that for any \( k = 2,\dots,K \) and any \( \alpha \leq \alpha' \), if \( f^{k-1}_{U_\alpha} \leq  f^{k}_{U_\alpha}\), then \( f_{U_{\alpha'}}(k-1) \leq  f_{U_{\alpha'}}(k)\). For any \( k=2,\dots,K \), we examine the following quantity:
\begin{align*}
&k(k-1)[g_{\alpha'}(k-1)g_{\alpha}(k) - g_{\alpha'}(k)g_{\alpha}(k-1)] \\
=&k(k-1)\left[ \left( \alpha' + \frac{1-\alpha'}{k-1} \right) \frac{1}{k-1} \left( \alpha + \frac{1-\alpha}{k} \right) \frac{1}{k} \right] \\
&- k(k-1) \left[ \left( \alpha' + \frac{1-\alpha'}{k} \right) \frac{1}{k} \left( \alpha + \frac{1-\alpha}{k-1} \right) \frac{1}{k-1} \right] \\
=&\left( \alpha' + \frac{1-\alpha'}{k-1} \right) \left( \alpha + \frac{1-\alpha}{k} \right) - \left( \alpha' + \frac{1-\alpha'}{k} \right) \left( \alpha + \frac{1-\alpha}{k-1} \right) \\
=&\left( \alpha\alpha' + \frac{\alpha'(1-\alpha)} {k} +\frac{\alpha(1-\alpha')}{k-1} + \frac{(1-\alpha)(1-\alpha')}{k(k-1)} \right) \\ 
&-\left( \alpha\alpha' + \frac{\alpha'(1-\alpha)}{k-1} + \frac{\alpha(1-\alpha')}{k} + \frac{(1-\alpha)(1-\alpha')}{k(k-1)} \right) \\
=&\left(\frac{\alpha'(1-\alpha)} {k} +\frac{\alpha(1-\alpha')}{k-1} \right) - \left( \frac{\alpha'(1-\alpha)}{k-1} + \frac{\alpha(1-\alpha')}{k} \right) \\
=&\left(\frac{\alpha'(1-\alpha) - \alpha(1 - \alpha')} {k} +\frac{\alpha(1-\alpha') - \alpha'(1-\alpha)}{k-1} \right) \\
=&\left(\frac{\alpha' - \alpha} {k} +\frac{\alpha - \alpha'}{k-1} \right) = \frac{k\alpha' - \alpha' - k\alpha + \alpha + k\alpha - k\alpha'}{k(k-1)} \\
=&\frac{\alpha - \alpha'}{k(k-1)} \leq 0, \alpha \leq \alpha' \land k = 2,\dots,K \\
\Rightarrow& g_{\alpha'}(k-1)g_{\alpha}(k) \leq g_{\alpha'}(k)g_{\alpha}(k-1).
\end{align*}
Therefore, with \(g_{\alpha}(k)\) is non-increasing and \(1/x\)-convex, we have
\begin{align*}
&\frac{g_{\alpha}(k)}{g_{\alpha}(k-1) - g_{\alpha}(k)} - \frac{g_{\alpha'}(k)}{g_{\alpha'}(k-1) - g_{\alpha'}(k)} \\
=&\frac{g_{\alpha'}(k-1)g_{\alpha}(k) - g_{\alpha'}(k)g_{\alpha}(k-1)}{(g_{\alpha}(k-1) - g_{\alpha}(k))(g_{\alpha'}(k-1) - g_{\alpha'}(k)) } \leq 0;
\end{align*}
or equivalently \[
\frac{g_{\alpha}(k)}{g_{\alpha}(k-1) - g_{\alpha}(k)} \leq \frac{g_{\alpha'}(k)}{g_{\alpha'}(k-1) - g_{\alpha'}(k)}.
\]
Therefore, we have the following relation
\begin{align*}
&f^{k-1}_{U_\alpha} \leq f^{k}_{U_\alpha} \Leftrightarrow g_{\alpha}(k-1) \sum_{k'=1}^{k-1} \mathbf{E}\left[ u(\bar{y}^{(k')},y) \right] \leq g_{\alpha}(k) \sum_{k'=1}^{k} \mathbf{E}\left[ u(\bar{y}^{(k')},y) \right] \\
\Leftrightarrow& \sum_{k'=1}^{k-1} \mathbf{E}\left[ u(\bar{y}^{(k')},y) \right] \leq \frac{g_{\alpha}(k)}{g_{\alpha}(k-1) - g_{\alpha}(k)} \mathbf{E}\left[ u(\bar{y}^{(k)},y) \right] \\
\Rightarrow& \sum_{k'=1}^{k-1} \mathbf{E}\left[ u(\bar{y}^{(k')},y) \right] \leq \frac{g_{\alpha'}(k)}{g_{\alpha'}(k-1) - g_{\alpha'}(k)} \mathbf{E}\left[ u(\bar{y}^{(k)},y) \right] \\
\Leftrightarrow& g_{\alpha'}(k-1) \sum_{k'=1}^{k-1} \mathbf{E}\left[ u(\bar{y}^{(k')},y) \right] \leq g_{\alpha'}(k) \sum_{k'=1}^{k} \mathbf{E}\left[ u(\bar{y}^{(k')},y) \right] \\
\Leftrightarrow& f^{k-1}_{U_{\alpha'}} \leq f^{k}_{U_{\alpha'}} \,.
\end{align*}
This implies that if \( \alpha \leq \alpha' \), then \( k_{\alpha}^* \leq k_{\alpha'}^* \). Therefore, we have \( \hat{Y}_{U_{\alpha}} = \{\hat{y}^{(1)}, \ldots \hat{y}^{(k_{\alpha}^*)}\} \subseteq \hat{Y}_{U_{\alpha'}} = \{\hat{y}^{(1)}, \ldots \hat{y}^{(k_{\alpha'}^*)}\}\). 
\end{proof}

\subsection{Proof for the Remark \ref{rem:upper-value-for-U-0/1}}
\begin{proof}
The proof is done by showing that the upper value $\overline{\alpha}_{u^{0/1}} = 3$. Since $s_k = 1$, $2 \leq k \leq K$, we have $\alpha_k = (k^2-1)/(k-1) = k+1$. Therefore, the smallest value of $\alpha_k$, $2 \leq k \leq K$, is $3$ when $k = 2$.      
\end{proof}

\subsection{Proof for Remark \ref{rem:probustness-permissiveness-for-U-0/1}}
\begin{proof}
For any singleton prediction $\bar{Y} = \{y\}$, we have $g_\alpha(|\bar{Y}|) = 1$, Therefore, we can simply pick up the degenerate distribution $\boldsymbol{p}(\mathcal{Y}\mid \bx)$ with $\boldsymbol{p}(y\mid \bx) = 1$ to have $\mathbf{E}_{\boldsymbol{p}(y\mid \bx)}\left[U^{0/1}_\alpha(\{y\},y) \right] = 1$, which in turn ensures that $\{y\}$ is a set-valued BOP on the distribution $\boldsymbol{p}(\mathcal{Y}\mid \bx)$. 

We now show that for any non-singleton $\bar{Y}$ with $|\bar{Y}| = k$, $k \geq 2$, $\bar{Y}$ is a set-valued BOP on the distribution $\boldsymbol{p}(\mathcal{Y}\mid \bx)$ where 
\begin{equation}
    \boldsymbol{p}(y\mid \bx) = 
    \begin{cases}
    \frac{1}{k} \text{ if } y \in \bar{Y} \,, \\
    0 \text{ otherwise.}
    \end{cases}
\end{equation}
In the case of 0/1 accuracy $u^{0/1}$, we have 
\begin{equation}
    \mathbf{E}_{\Prob(\cY \given \bx)}\left[u^{0/1}(\bar{y},y) \right] = \boldsymbol{p}(\bar{y} \mid \bx)=
    \begin{cases}
    \frac{1}{k} \text{ if } \bar{y} \in \bar{Y} \,, \\
    0 \text{ otherwise.}
    \end{cases}
\end{equation}
Therefore, we have 
\begin{align*}
    \hat{Y}_{U^{0/1}_{\alpha}} &\in \operatorname*{argmax}_{\bar{Y} \in \cY_{\mathrm{set}}} \mathbf{E}_{\Prob(\cY \given \bx)}\left[U^{0/1}_\alpha(\bar{Y},y) \right] \\
    &= \operatorname*{argmax}_{\bar{Y} \in \cY_{\mathrm{set}}} g_\alpha(|\bar{Y}|)\sum_{y\in \cY} \llbracket y\in \bar{Y}\rrbracket\mathbf{E}_{\Prob(\cY \given \bx)}\left[u^{0/1}(\bar{Y},y) \right]\\
    & = \operatorname*{argmax}_{\bar{Y} \in \cY_{\mathrm{set}}} g_\alpha(|\bar{Y}|)\sum_{y\in \bar{Y} \cap \hat{Y}} \mathbf{E}_{\Prob(\cY \given \bx)}\left[u^{0/1}(\bar{y},y) \right] \\
    & = \operatorname*{argmax}_{\bar{Y} \subset \hat{Y}} g_\alpha(|\bar{Y}|) \frac{|\bar{Y}|}{k} \\
    & = \operatorname*{argmax}_{\bar{Y} \subset \hat{Y}} \left(\left(\alpha + \frac{1 - \alpha}{|\bar{Y}|}\right) \frac{1}{|\bar{Y}|}\right)\frac{|\bar{Y}|}{k} \\
    & = \operatorname*{argmax}_{\bar{Y} \subset \hat{Y}}\left(\alpha + \frac{1 - \alpha}{|\bar{Y}|}\right)\frac{1}{k} 
\end{align*}
because $\alpha + (1 - \alpha)/|\bar{Y}|$ with $\alpha > 1$ is a increasing function of $|\bar{Y}|$. 
   
\end{proof}

\section{Additional experimental results} \label{sec:Additional experimental results}

This section provides detailed experimental results for experiments that are mentioned in Section \ref{sec:Additional experimental results}. 

\subsection{Additional experimental results}

Results when the BNNs are employed to estimate the representative distributions \eqref{eq:representative distribution}, which are then used to define set-valued BOPs \eqref{eq:set-valued BOP} of the utility-discounted accuracies $U^{0/1}_\alpha$ \eqref{eq:U-01-alpha} are given in Figs.~\ref{fig:BNNs with clean training data no sensitivity}--\ref{fig:BNNs with noisy training data no sensitivity}. The trends observed on all the experiments are similar to those that are discussed in Section \ref{sec:Results}. 

The results when the CNNs with the dropout prediction option are employed to estimate the representative distributions \eqref{eq:representative distribution}, which are then used to define set-valued BOPs \eqref{eq:set-valued BOP} of the reward-sensitive utility-discounted accuracies $U_\alpha$ \eqref{eq:Reward-sensitive-U-alpha} are given in Figs.\ref{fig:CNNs with clean training data}--\ref{fig:CNNs with noisy training data}. Again, the trends are similar to those that are discussed in Section \ref{sec:Results}, except that CNNs appear to be (much) less promising when encountering the mismatches between training and test data types. The trends are visible on all three data sets, especially on the Plant Leaf Diseases data set. Again, CNNs trained on noisy images tend to be more robust compared to those that are trained on clean images.       

Finally, the results when the CNNs with the dropout prediction option are employed to estimate the representative distributions \eqref{eq:representative distribution}, which are then used to define set-valued BOPs \eqref{eq:set-valued BOP} of the utility-discounted accuracies $U^{0/1}_\alpha$ \eqref{eq:U-01-alpha} are given in Figs.~\ref{fig:CNNs with clean training data no sensitivity}--\ref{fig:CNNs with noisy training data no sensitivity}. The trends in the results are similar to those that are observed for the case of optimizing the reward-sensitive utility-discounted accuracies $U_\alpha$ \eqref{eq:Reward-sensitive-U-alpha}.

Overall, the results suggest that our proposal on making set-valued BOPs can provide a promising way to seek the trade-off between recall and precision on data sets of different application domains, when the reward-sensitivity can be optionally taken into account when assessing the predictions. Our proposal might provide rather impressive trade-offs when the training data and test data are of the same type, and when BNNs are trained on clean data, but both clean and noisy test data can be encountered.     

\input{tikz/bnn_train_clean_nonsense}
\input{tikz/bnn_train_gnoise_nonsense}

\input{tikz/cdo_train_clean}
\input{tikz/cdo_train_gnoise}

\input{tikz/cdo_train_clean_nonsense}
\input{tikz/cdo_train_gnoise_nonsense}

\subsection{Calibration errors}

To seek further insights about the potential impact of the quality of probabilistic predictions achieved by our representative probabilities, we report in this appendix different (the lower the better) calibration errors \citep{silva2023classifier}, which are computed using the representative probability distributions \eqref{eq:representative distribution}.
\begin{itemize}
    \item $\mathrm{ECE}_{\mathrm{classwise}}$: Let us first discretize the score space $[0,1]$ into $M$ bins $\{\mathrm{B}_1, \ldots, \mathrm{B}_M\}$. For any $k \in [K]$, let 
\begin{align}
    \bar{y}^k_m &\defi \frac{1}{|\mathrm{B}_m|}\sum_{t=1}^T \llbracket \Prob(y^k \given \bx_t) \in \mathrm{B}_m \rrbracket \llbracket y_t = y^k\rrbracket \,, m \in [M]\,,\\
    \bar{s}^k_m &\defi \frac{1}{|\mathrm{B}_m|}\sum_{t=1}^T \llbracket \Prob(y^k \given \bx_t) \in \mathrm{B}_m \rrbracket \Prob(y^k \given \bx_t) \,, m \in [M]\,.  
\end{align} 
The $\mathrm{ECE}_k$ is defined as follows
\begin{equation} \label{eq:Binary-ECE}
    \mathrm{ECE}_k(\bh \given \cD) = \sum_{m=1}^M \frac{|\mathrm{B}_m|}{T} |\bar{y}^k_m - \bar{s}^k_m | \,.  
\end{equation} 
The estimated classwise calibration error is defined as follows
    \begin{equation}\label{eq:ECE-classwise}
       \mathrm{ECE}_{\mathrm{classwise}}(\bh \given \mathbf{T}) = \frac{1}{K} \sum_{k=1}^K\mathrm{ECE}_k(\bh \given \mathbf{T}) \,,
    \end{equation}
   \item $\mathrm{MCE}_{\mathrm{classwise}}$: The maximum classwise calibration error is defined as follows
   \begin{equation}\label{eq:MCE-classwise}
       \mathrm{MCE}_{\mathrm{classwise}}(\bh \given \mathbf{T})  = \max_{k \in [K], m \in [M]} |\bar{y}^k_m - \bar{s}^k_m |\,.
   \end{equation}
    \item $\mathrm{ECE}_{\mathrm{confidence}}$: For any $m\in [M]$, let 
    \begin{align}
    \bar{y}^m_{\mathrm{confidence}} &\defi \frac{1}{|\mathrm{B}_m|}\sum_{t=1}^T \llbracket \max_{y \in \cY}\Prob(y \given \bx_t) \in \mathrm{B}_m \rrbracket \llbracket  \operatorname*{argmax}_{y \in \cY}\Prob(y \given \bx_t) = y_t \rrbracket \,, \\
    \bar{s}^m_{\mathrm{confidence}} &\defi \frac{1}{|\mathrm{B}_m|}\sum_{t=1}^T \llbracket \max_{y \in \cY}\Prob(y \given \bx_t) \in \mathrm{B}_m \rrbracket \max_{y \in \cY} \Prob(y \given \bx_t) \,.  
   \end{align}
   The estimated Confidence calibration error is defined as follows
   \begin{equation}\label{eq:ECE-confidence}
       \mathrm{ECE}_{\mathrm{confidence}}(\bh \given \mathbf{T}) = \sum_{m=1}^M \frac{|\mathrm{B}_m|}{T} |\bar{y}^m_{\mathrm{confidence}} - \bar{s}^m_{\mathrm{confidence}}| \,.
   \end{equation}
   \item $\mathrm{MCE}_{\mathrm{confidence}}$: The maximum confidence calibration error is defined as follows 
   \begin{equation}\label{eq:MCE-confidence}
        \mathrm{MCE}_{\mathrm{confidence}}(\bh \given \mathbf{T}) =  \max_{m \in [M]} |\bar{y}^m_{\mathrm{confidence}} - \bar{s}^m_{\mathrm{confidence}}| \,.
   \end{equation}
\end{itemize}

Calibration errors of the representative distributions \eqref{eq:representative distribution} produced by $d_{\mathrm{sE}}$ \eqref{eq:squared Euclidean distance}, $d_{\mathrm{L1}}$ \eqref{eq:L_1 distance}, and $d_{\mathrm{KL}}$ \eqref{eq:KL divergence} are given in Table \ref{table:calibration_error_sE}, \ref{table:calibration_error_L1}, and \ref{table:calibration_error_KL}, respectively. 

\begin{table}[http]
\centering
\caption{Calibration errors of the representative distributions \eqref{eq:representative distribution} produced by $d_{\mathrm{sE}}$ \eqref{eq:squared Euclidean distance}: The shorthand notations $\mathrm{ECE}_{\mathrm{cl}}$, $\mathrm{MCE}_{\mathrm{cl}}$, $\mathrm{ECE}_{\mathrm{co}}$, and $\mathrm{MCE}_{\mathrm{co}}$ are used instead of $\mathrm{ECE}_{\mathrm{classwise}}$, $\mathrm{MCE}_{\mathrm{classwise}}$, $\mathrm{ECE}_{\mathrm{confidence}}$, and $\mathrm{MCE}_{\mathrm{confidence}}$. Results on (clean train data, clean test data), (clean train data, noisy test data), (noisy train data, clean test data), and (noisy train data, noisy test data) are marked as c-c, c-n, n-c, and n-n, respectively.}
\label{table:calibration_error_sE}
\begin{tabular}{p{1em}|p{3em}|p{4.6em}p{4.6em}|p{4.6em}p{4.6em}|p{4.6em}p{4.6em}}
\multicolumn{2}{c|}{}                                & \multicolumn{2}{c|}{CIFAR-10} & \multicolumn{2}{|c|}{Fashion-MNIST} & \multicolumn{2}{|c}{Plan Leaf Diseases} \\
\cline{3-8}
\multicolumn{2}{c|}{}                                & BNNs           & CNNs          & BNNs          & CNNS         & BNN         & CNNs        \\
\toprule
\multirow{4}{*}{c-n} & $\mathrm{ECE}_{\mathrm{cl}}$  & 7.99\(\pm\)0.38           & 8.75\(\pm\)0.86          & 2.82\(\pm\)0.04          & 6.62\(\pm\)0.06         & 2.45\(\pm\)0.14         & 3.68\(\pm\)0.13        \\
\cline{2-8}
                      & $\mathrm{MCE}_{\mathrm{cl}}$  & 71.11\(\pm\)3.55           & 70.80\(\pm\)5.51          & 59.30\(\pm\)11.09          & 77.87\(\pm\)1.42         & 97.02\(\pm\)0.34         & 95.74\(\pm\)0.62        \\
\cline{2-8}
                      & $\mathrm{ECE}_{\mathrm{co}}$ & 26.53\(\pm\)5.06           & 31.46\(\pm\)5.69          & 9.50\(\pm\)0.51          & 5.33\(\pm\)0.62         & 11.53\(\pm\)2.69         & 33.20\(\pm\)3.91        \\
\cline{2-8}
                      & $\mathrm{MCE}_{\mathrm{co}}$ & 43.17\(\pm\)10.86           & 51.74\(\pm\)8.57          & 50.80\(\pm\)25.43          & 21.52\(\pm\)6.76         & 18.59\(\pm\)3.86         & 61.42\(\pm\)2.99        \\
\midrule
\multirow{4}{*}{c-c} & $\mathrm{ECE}_{\mathrm{cl}}$  & 0.90\(\pm\)0.04           & 1.29\(\pm\)0.03          & 0.40\(\pm\)0.02          & 3.48\(\pm\)0.08         & 0.50\(\pm\)0.01         & 1.52\(\pm\)0.03        \\
\cline{2-8}
                      & $\mathrm{MCE}_{\mathrm{cl}}$  & 32.56\(\pm\)2.39           & 32.23\(\pm\)1.36         & 63.29\(\pm\)11.37          & 53.46\(\pm\)1.96         & 84.22\(\pm\)5.61         & 77.23\(\pm\)0.70        \\
\cline{2-8}
                      & $\mathrm{ECE}_{\mathrm{co}}$ & 3.87\(\pm\)0.25           & 5.79\(\pm\)0.45          & 1.23\(\pm\)0.19          & 12.72\(\pm\)0.10         & 9.62\(\pm\)0.09         & 29.23\(\pm\)0.47        \\
\cline{2-8}
                      & $\mathrm{MCE}_{\mathrm{co}}$ & 38.94\(\pm\)28.25           & 43.33\(\pm\)28.60          & 22.50\(\pm\)4.20          & 21.00\(\pm\)0.35         & 53.58\(\pm\)19.61         & 59.01\(\pm\)6.00        \\
\midrule
\multirow{4}{*}{n-n} & $\mathrm{ECE}_{\mathrm{cl}}$  & 0.86\(\pm\)0.02           & 1.19\(\pm\)0.00          & 0.44\(\pm\)0.02          & 2.98\(\pm\)0.03         & 0.56\(\pm\)0.01         & 1.69\(\pm\)0.02       \\
\cline{2-8}
                      & $\mathrm{MCE}_{\mathrm{cl}}$  & 25.63\(\pm\)2.18           & 25.84\(\pm\)0.25         & 53.23\(\pm\)9.75          & 50.15\(\pm\)2.67         & 78.03\(\pm\)6.04         & 79.55\(\pm\)1.31        \\
\cline{2-8}
                      & $\mathrm{ECE}_{\mathrm{co}}$ & 2.88\(\pm\)0.17           & 4.99\(\pm\)0.06          & 1.51\(\pm\)0.13          & 11.41\(\pm\)0.30         & 10.72\(\pm\)0.21         & 32.74\(\pm\)0.57        \\
\cline{2-8}
                      & $\mathrm{MCE}_{\mathrm{co}}$ & 11.27\(\pm\)4.90           & 15.01\(\pm\)4.16          & 17.10\(\pm\)2.73          & 40.05\(\pm\)28.41         & 57.54\(\pm\)18.88         & 57.13\(\pm\)1.86        \\
\midrule
\multirow{4}{*}{n-c} & $\mathrm{ECE}_{\mathrm{cl}}$  & 1.61\(\pm\)0.25           & 2.74\(\pm\)0.19          & 0.52\(\pm\)0.03          & 3.69\(\pm\)0.07         & 0.84\(\pm\)0.03         & 2.26\(\pm\)0.10        \\
\cline{2-8}
                      & $\mathrm{MCE}_{\mathrm{cl}}$  & 54.81\(\pm\)0.47           & 67.26\(\pm\)2.89         & 66.35\(\pm\)1.83          & 68.17\(\pm\)0.59         & 87.54\(\pm\)3.52         & 89.51\(\pm\)2.97        \\
\cline{2-8}
                      & $\mathrm{ECE}_{\mathrm{co}}$ & 1.48\(\pm\)0.24           & 1.35\(\pm\)0.12          & 1.57\(\pm\)0.16          & 2.38\(\pm\)0.20         & 9.12\(\pm\)0.65         & 6.61\(\pm\)2.93        \\
\cline{2-8}
                      & $\mathrm{MCE}_{\mathrm{co}}$ & 20.30\(\pm\)6.35           & 16.83\(\pm\)2.65          & 20.82\(\pm\)4.90          & 40.98\(\pm\)30.37         & 43.17\(\pm\)29.10         & 14.90\(\pm\)3.77        \\
\bottomrule                      
\end{tabular}
\end{table}

\begin{table}[http]
\centering
\caption{Calibration errors of the representative distributions \eqref{eq:representative distribution} produced by $d_{\mathrm{L1}}$ \eqref{eq:L_1 distance}: The shorthand notations $\mathrm{ECE}_{\mathrm{cl}}$, $\mathrm{MCE}_{\mathrm{cl}}$, $\mathrm{ECE}_{\mathrm{co}}$, and $\mathrm{MCE}_{\mathrm{co}}$ are used instead of $\mathrm{ECE}_{\mathrm{classwise}}$, $\mathrm{MCE}_{\mathrm{classwise}}$, $\mathrm{ECE}_{\mathrm{confidence}}$, and $\mathrm{MCE}_{\mathrm{confidence}}$. Results on (clean train data, clean test data), (clean train data, noisy test data), (noisy train data, clean test data), and (noisy train data, noisy test data) are marked as c-c, c-n, n-c, and n-n, respectively.}
\label{table:calibration_error_L1}
\begin{tabular}{p{1em}|p{3em}|p{4.6em}p{4.6em}|p{4.6em}p{4.6em}|p{4.6em}p{4.6em}}
\multicolumn{2}{c|}{}                                & \multicolumn{2}{c|}{CIFAR-10} & \multicolumn{2}{|c|}{Fashion-MNIST} & \multicolumn{2}{|c}{Plan Leaf Diseases} \\
\cline{3-8}
\multicolumn{2}{c|}{}                                & BNNs           & CNNs          & BNNs          & CNNS         & BNN         & CNNs        \\
\toprule
\multirow{4}{*}{c-n} & $\mathrm{ECE}_{\mathrm{cl}}$  & 9.35\(\pm\)0.64          & 11.57\(\pm\)0.88          & 2.59\(\pm\)0.04          & 5.79\(\pm\)0.13         & 2.49\(\pm\)0.16         & 4.23\(\pm\)0.08        \\
\cline{2-8}
                      & $\mathrm{MCE}_{\mathrm{cl}}$  & 81.42\(\pm\)4.37           & 88.76\(\pm\)2.33          & 68.29\(\pm\)3.99          & 81.41\(\pm\)0.89         & 95.52\(\pm\)0.29         & 98.22\(\pm\)1.14        \\
\cline{2-8}
                      & $\mathrm{ECE}_{\mathrm{co}}$ & 38.59\(\pm\)5.24           & 51.72\(\pm\)5.25          & 8.03\(\pm\)0.54          & 5.06\(\pm\)0.87         & 37.21\(\pm\)2.55         & 62.66\(\pm\)3.21        \\
\cline{2-8}
                      & $\mathrm{MCE}_{\mathrm{co}}$ & 61.86\(\pm\)14.26           & 68.27\(\pm\)4.53          & 42.55\(\pm\)29.52          & 24.81\(\pm\)16.83         & 50.94\(\pm\)2.12         & 82.51\(\pm\)3.14        \\
\midrule
\multirow{4}{*}{c-c} & $\mathrm{ECE}_{\mathrm{cl}}$  & 0.78\(\pm\)0.01          & 1.07\(\pm\)0.07          & 0.46\(\pm\)0.02          & 2.08\(\pm\)0.09         & 0.07\(\pm\)0.00         & 0.30\(\pm\)0.00        \\
\cline{2-8}
                      & $\mathrm{MCE}_{\mathrm{cl}}$  & 31.46\(\pm\)2.93           & 38.22\(\pm\)2.22          & 66.79\(\pm\)8.35          & 51.57\(\pm\)4.97        & 88.10\(\pm\)3.72         & 80.45\(\pm\)0.46        \\
\cline{2-8}
                      & $\mathrm{ECE}_{\mathrm{co}}$ & 2.57\(\pm\)0.07           & 4.81\(\pm\)0.35          & 1.51\(\pm\)0.15          & 3.01\(\pm\)0.17         & 0.96\(\pm\)0.04         & 4.80\(\pm\)0.17        \\
\cline{2-8}
                      & $\mathrm{MCE}_{\mathrm{co}}$ & 24.51\(\pm\)2.57           & 41.09\(\pm\)23.36         & 27.24\(\pm\)12.68          & 17.39\(\pm\)2.47         & 54.70\(\pm\)22.95         & 35.99\(\pm\)6.16       \\
\midrule
\multirow{4}{*}{n-n} & $\mathrm{ECE}_{\mathrm{cl}}$  & 1.05\(\pm\)0.02          & 1.74\(\pm\)0.03          & 0.52\(\pm\)0.03          & 1.85\(\pm\)0.09         & 0.11\(\pm\)0.01         & 0.37\(\pm\)0.02        \\
\cline{2-8}
                      & $\mathrm{MCE}_{\mathrm{cl}}$  & 29.05\(\pm\)3.58           & 42.69\(\pm\)1.82          & 70.16\(\pm\)1.33          & 50.67\(\pm\)1.08        & 88.62\(\pm\)3.75         & 84.93\(\pm\)3.64        \\
\cline{2-8}
                      & $\mathrm{ECE}_{\mathrm{co}}$ & 3.82\(\pm\)0.12           & 7.73\(\pm\)0.14          & 1.85\(\pm\)0.13          & 3.69\(\pm\)0.19         & 1.54\(\pm\)0.16         & 6.00\(\pm\)0.20        \\
\cline{2-8}
                      & $\mathrm{MCE}_{\mathrm{co}}$ & 17.25\(\pm\)0.83           & 25.83\(\pm\)1.73         & 18.81\(\pm\)1.06          & 14.5\(\pm\)4.40         & 40.42\(\pm\)13.56         & 67.77\(\pm\)20.59       \\
\midrule
\multirow{4}{*}{n-c} & $\mathrm{ECE}_{\mathrm{cl}}$  & 1.78\(\pm\)0.20          & 3.21\(\pm\)0.20          & 0.56\(\pm\)0.03          & 3.31\(\pm\)0.12         & 0.49\(\pm\)0.04         & 2.00\(\pm\)0.23        \\
\cline{2-8}
                      & $\mathrm{MCE}_{\mathrm{cl}}$  & 52.34\(\pm\)1.61           & 74.79\(\pm\)3.10          & 70.60\(\pm\)12.32          & 78.84\(\pm\)2.22        & 93.24\(\pm\)1.31         & 94.08\(\pm\)1.41        \\
\cline{2-8}
                      & $\mathrm{ECE}_{\mathrm{co}}$ & 5.16\(\pm\)0.36           & 12.80\(\pm\)0.61          & 1.67\(\pm\)0.12          & 5.07\(\pm\)0.32         & 2.12\(\pm\)0.66         & 18.18\(\pm\)5.42        \\
\cline{2-8}
                      & $\mathrm{MCE}_{\mathrm{co}}$ & 26.92\(\pm\)14.57           & 29.77\(\pm\)4.40         & 38.91\(\pm\)27.55          & 40.70\(\pm\)30.33         & 17.97\(\pm\)6.26         & 29.41\(\pm\)6.91       \\
\bottomrule                      
\end{tabular}
\end{table}

Overall, the results suggest that, for all the distances, BNNs tend to produce more calibrated representative distributions \eqref{eq:representative distribution}, compared to those that are produced by CNNs with the dropout prediction option. This might explain why BNNs tend to be more robust when encountering the mismatches between training and test data types. 

Moreover, the calibration errors, especially $\mathrm{ECE}_{\mathrm{classwise}}$ \eqref{eq:ECE-classwise} and $\mathrm{ECE}_{\mathrm{confidence}}$ \eqref{eq:ECE-confidence}, seem to correlate with the various results on the predictive performance. First, the calibration errors are significantly higher when encountering the mismatches between training and test data types. This might suggest an explanation for the downgrade of the predictive performance on these cases.   
\clearpage

\begin{table}[http]
\centering
\caption{Calibration errors of the representative distributions \eqref{eq:representative distribution} produced by $d_{\mathrm{KL}}$ \eqref{eq:KL divergence}: The shorthand notations $\mathrm{ECE}_{\mathrm{cl}}$, $\mathrm{MCE}_{\mathrm{cl}}$, $\mathrm{ECE}_{\mathrm{co}}$, and $\mathrm{MCE}_{\mathrm{co}}$ are used instead of $\mathrm{ECE}_{\mathrm{classwise}}$, $\mathrm{MCE}_{\mathrm{classwise}}$, $\mathrm{ECE}_{\mathrm{confidence}}$, and $\mathrm{MCE}_{\mathrm{confidence}}$. Results on (clean train data, clean test data), (clean train data, noisy test data), (noisy train data, clean test data), and (noisy train data, noisy test data) are marked as c-c, c-n, n-c, and n-n, respectively.}
\label{table:calibration_error_KL}
\begin{tabular}{p{1em}|p{3em}|p{4.6em}p{4.6em}|p{4.6em}p{4.6em}|p{4.6em}p{4.6em}}
\multicolumn{2}{c|}{}                                & \multicolumn{2}{c|}{CIFAR-10} & \multicolumn{2}{|c|}{Fashion-MNIST} & \multicolumn{2}{|c}{Plan Leaf Diseases} \\
\cline{3-8}
\multicolumn{2}{c|}{}                                & BNNs           & CNNs          & BNNs          & CNNS         & BNN         & CNNs        \\
\toprule
\multirow{4}{*}{c-n} & $\mathrm{ECE}_{\mathrm{cl}}$  & 9.39\(\pm\)0.62           & 11.67\(\pm\)0.87          & 2.46\(\pm\)0.05          & 5.97\(\pm\)0.12         & 2.62\(\pm\)0.17         & 4.25\(\pm\)0.08        \\
\cline{2-8}
                      & $\mathrm{MCE}_{\mathrm{cl}}$  & 82.07\(\pm\)5.07           & 90.11\(\pm\)1.77          & 67.95\(\pm\)4.12          & 82.45\(\pm\)0.81         & 97.10\(\pm\)2.05         & 99.27\(\pm\)0.49        \\
\cline{2-8}
                      & $\mathrm{ECE}_{\mathrm{co}}$ & 39.58\(\pm\)5.04           & 52.73\(\pm\)5.14          & 7.35\(\pm\)0.51          & 6.68\(\pm\)0.69         & 41.77\(\pm\)2.87         & 65.71\(\pm\)3.71        \\
\cline{2-8}
                      & $\mathrm{MCE}_{\mathrm{co}}$ & 56.61\(\pm\)6.75           & 68.20\(\pm\)3.29          & 41.64\(\pm\)30.06          & 17.77\(\pm\)1.82         & 57.19\(\pm\)1.94         & 83.65\(\pm\)2.30        \\
\midrule
\multirow{4}{*}{c-c} & $\mathrm{ECE}_{\mathrm{cl}}$  & 0.88\(\pm\)0.02           & 1.15\(\pm\)0.09          & 0.36\(\pm\)0.02          & 2.16\(\pm\)0.09         & 0.06\(\pm\)0.00         & 0.26\(\pm\)0.01        \\
\cline{2-8}
                      & $\mathrm{MCE}_{\mathrm{cl}}$  & 39.12\(\pm\)3.95           & 46.76\(\pm\)1.92          & 69.82\(\pm\)9.63          & 56.14\(\pm\)1.83         & 92.11\(\pm\)2.23         & 86.73\(\pm\)1.37        \\
\cline{2-8}
                      & $\mathrm{ECE}_{\mathrm{co}}$ & 3.69\(\pm\)0.05           & 5.36\(\pm\)0.39          & 1.00\(\pm\)0.20          & 2.86\(\pm\)0.17         & 0.42\(\pm\)0.02         & 3.65\(\pm\)0.10        \\
\cline{2-8}
                      & $\mathrm{MCE}_{\mathrm{co}}$ & 22.04\(\pm\)0.19           & 77.20\(\pm\)1.39          & 18.43\(\pm\)9.01          & 19.53\(\pm\)9.65         & 35.41\(\pm\)2.64         & 32.39\(\pm\)1.36        \\
\midrule
\multirow{4}{*}{n-n} & $\mathrm{ECE}_{\mathrm{cl}}$  & 1.15\(\pm\)0.01           & 1.85\(\pm\)0.03          & 0.41\(\pm\)0.02          & 1.85\(\pm\)0.10         & 0.08\(\pm\)0.01         & 0.30\(\pm\)0.02        \\
\cline{2-8}
                      & $\mathrm{MCE}_{\mathrm{cl}}$  & 31.76\(\pm\)1.26           & 45.42\(\pm\)4.07          & 59.87\(\pm\)1.82          & 52.15\(\pm\)1.51         & 91.61\(\pm\)3.13         & 85.94\(\pm\)3.04        \\
\cline{2-8}
                      & $\mathrm{ECE}_{\mathrm{co}}$ & 4.79\(\pm\)0.10           & 8.56\(\pm\)0.08          & 1.15\(\pm\)0.03          & 3.52\(\pm\)0.23         & 0.57\(\pm\)0.10         & 4.52\(\pm\)0.26        \\
\cline{2-8}
                      & $\mathrm{MCE}_{\mathrm{co}}$ & 19.85\(\pm\)6.55           & 24.98\(\pm\)0.70          & 16.50\(\pm\)2.77          & 46.33\(\pm\)24.07         & 45.36\(\pm\)18.46         & 46.58\(\pm\)24.42        \\
\midrule
\multirow{4}{*}{n-c} & $\mathrm{ECE}_{\mathrm{cl}}$  & 1.89\(\pm\)0.21           & 3.30\(\pm\)0.17          & 0.67\(\pm\)0.02          & 3.33\(\pm\)0.13         & 0.49\(\pm\)0.05         & 2.03\(\pm\)0.23        \\
\cline{2-8}
                      & $\mathrm{MCE}_{\mathrm{cl}}$  & 52.45\(\pm\)1.95           & 73.36\(\pm\)2.44          & 70.58\(\pm\)7.20          & 80.71\(\pm\)1.09         & 93.84\(\pm\)1.66         & 94.11\(\pm\)1.75        \\
\cline{2-8}
                      & $\mathrm{ECE}_{\mathrm{co}}$ & 6.51\(\pm\)0.48           & 13.93\(\pm\)0.53          & 2.70\(\pm\)0.16          & 5.30\(\pm\)0.46         & 4.21\(\pm\)0.84         & 21.74\(\pm\)5.60        \\
\cline{2-8}
                      & $\mathrm{MCE}_{\mathrm{co}}$ & 38.31\(\pm\)27.35           & 34.98\(\pm\)4.79          & 35.48\(\pm\)29.87          & 19.31\(\pm\)1.07         & 41.62\(\pm\)26.30         & 34.41\(\pm\)6.29        \\
\bottomrule                      
\end{tabular}
\end{table}

\end{appendices}


\bibliography{paper-bibliography}

\end{document}